\RequirePackage{etoolbox}
\patchcmd{\bibliographystyle}{#1}{myplainnat}{}{}
\documentclass{colt2017} 

\usepackage{renewtheorem}
\usepackage[font={small}]{caption}

\usepackage[utf8]{inputenc} 
\usepackage[T1]{fontenc}    

\usepackage{color}
\usepackage{pdfsync}

\newtheorem{fact}{Fact}

\usepackage{etoolbox,xparse,xspace,setspace}
\usepackage{array,multirow,booktabs} 
\usepackage{bigstrut} 
\usepackage{bbm} 
\usepackage{wrapfig}

\graphicspath{{figs/}}

\usepackage{tikz,pgfplots,tikz-3dplot}
\usetikzlibrary{arrows,shapes,positioning,calc,intersections,through}
\tdplotsetmaincoords{55}{135}        
\tikzstyle{cell}=[dashed,thick]
\tikzstyle{simplex}=[thick]
\newcommand{\placefiglabel}[1]{\node at (.9,0,1.4) {#1};}

\usepackage{pifont}
\newcommand{\cmark}{\textcolor{darkgreen}{\ding{51}}}
\newcommand{\xmark}{\textcolor{red}{\ding{55}}}

\newcommand{\Comments}{1}
\newcommand{\FutureComments}{1}
\definecolor{gray}{gray}{0.5}
\definecolor{darkgreen}{rgb}{0,0.5,0}
\definecolor{darkorange}{rgb}{.8,0.4,0}
\newcommand{\mynote}[2]{\ifnum\Comments=1\textcolor{#1}{#2}\fi}

\newcommand{\future}[1]{\ifnum\FutureComments=1\textcolor{gray}{[FUTURE: #1]}\fi}

\newcommand{\Var}{\mathop{\mathrm{Var}}}

\newcommand{\loss}{\ell}

\newcommand{\D}{\mathcal{D}}
\newcommand{\E}{\mathbb{E}}
\newcommand{\F}{\mathcal{F}}

\renewcommand{\O}{\mathcal{O}}
\renewcommand{\P}{\mathcal{P}}

\newcommand{\R}{\mathcal{R}}
\let\oldS\S 
\newcommand{\sect}{\mbox{\oldS\hspace{-.1mm}}}
\renewcommand{\S}{\mathcal{S}}

\newcommand{\Y}{\mathbb{Y}}
\newcommand{\X}{\mathcal{X}}

\renewcommand{\vec}[1]{{\mathbf{#1}}}

\renewcommand{\o}{\mathit{o}}

\newcommand{\x}{\vec{x}}
\newcommand{\y}{\vec{y}}

\newcommand{\conv}{\convhull}

\newcommand{\toto}{\rightrightarrows}

\newcommand{\ones}{\mathbbm{1}}

\def\reals{\mathbb{R}}

\newcommand{\argmin}{\mathop{\mathrm{argmin}}}

\usepackage{mathtools}
\DeclarePairedDelimiter{\ceil}{\lceil}{\rceil}
 \renewcommand{\ones}{\mathbbm{1}} 
\renewcommand{\conv}{\mathrm{conv}} 
\renewcommand{\Y}{\mathcal{Y}} 
\renewcommand{\O}{\Y} 
\renewcommand{\o}{y}

\renewcommand{\Comments}{0}
\renewcommand{\FutureComments}{0}

\title[Multi-Observation Elicitation]{Multi-Observation Elicitation}
\usepackage{times}
\coltauthor{
  \Name{Sebastian {Casalaina-Martin}} \Email{casa@math.colorado.edu}\\
  \addr CU Boulder
  \AND
  \Name{Rafael Frongillo} \Email{raf@colorado.edu}\\
  \addr CU Boulder
  \AND
  \Name{Tom Morgan} \Email{tdmorgan@seas.harvard.edu}\\
  \addr Harvard
  \AND
  \Name{Bo Waggoner} \Email{bwag@seas.upenn.edu}\\
  \addr UPenn
}

\newif\ifHideFoot
\HideFoottrue  
\HideFootfalse  

\ifHideFoot
\usepackage[cmtip,all,matrix,arrow,tips,curve]{xy}
\newcommand{\Yano}[1]{}
 
\else 
\newcommand{\marg}[1]{\normalsize{{
\color{red}\footnote{{\color{blue}#1}}}
{
\marginpar[\vskip -.25cm \hskip -.8 cm {\color{red}\thefootnote$\implies$}]
{\vskip
-.2cm{\color{red}$\impliedby$\thefootnote}}}}}
\newcommand{\Yano}[1]{\marg{(Yano) #1}}

\usepackage{amssymb,amsmath,amscd,mathrsfs,graphicx, color, mathtools, enumerate}
\usepackage[cmtip,all,matrix,arrow,tips,curve]{xy}
\fi

\newif\ifHideDiag
\HideDiagtrue

\begin{document}

\maketitle

\begin{abstract}
  We study loss functions that measure the accuracy of a prediction based on multiple data points simultaneously.
  To our knowledge, such loss functions have not been studied before in the area of property elicitation or in machine learning more broadly.
  As compared to traditional loss functions that take only a single data point, these multi-observation loss functions can in some cases drastically reduce the dimensionality of the hypothesis required.
  In elicitation, this corresponds to requiring many fewer reports; in empirical risk minimization, it corresponds to algorithms on a hypothesis space of much smaller dimension.
  We explore some examples of the tradeoff between dimensionality and number of observations, give some geometric characterizations and intuition for relating loss functions and the properties that they elicit, and discuss some implications for both elicitation and machine-learning contexts.
\end{abstract}


\section{Introduction}
\label{sec:introduction}

In machine learning and statistics, empirical risk minimization (ERM) is a dominant inference technique, wherein a model is chosen which minimizes some loss function over a data set.  As the choice of loss function used in ERM may have a large impact on the model chosen, how should one choose this loss?  A growing body of work in \emph{property elicitation} seeks to answer this question, by viewing a loss function as ``incentivizing'' the prediction of a particular conditional statistic~\citep{lambert2008eliciting,gneiting2011making,steinwart2014elicitation,frongillo2015vector-valued,agarwal2015consistent}; for example, it is well-known that squared loss elicits the mean, and hence least-squares regression finds the best fit to the conditional means of the data.\footnote{There are also contributions from microeconomics, and crowdsourcing in particular, where one wishes to incentivize humans rather than algorithms, but the mathematics is the same.}

A natural question, which is still open in the vector-valued case, is the following: for which conditional statistics do there exist loss functions which elicit them?  Positive examples include the mean, median, other quantiles, moments, and several others.  Perhaps surprisingly, however, there are negative examples as well: it is well-known that the variance is not elicitable, meaning there is no loss function for which minimizing the loss will yield the variance of the data or distribution.

The usual approach to dealing with non-elicitable statistics is called \emph{indirect elicitation}: elicit other conditional statistics from which one can compute the desired statistic.  For example, the variance of a distribution can be written as (2nd moment) - (1st moment)$^2$, and as mentioned above, moments are elicitable.  The question of how many such auxiliary statistics are required gives rise to the concept of \emph{elicitation complexity}; since the variance cannot be elicited with one but can with two, we say it is 2-elicitable~\citep{lambert2008eliciting,frongillo2015elicitation-2}.

In this paper, we explore an alternative approach to dealing with non-elicitable statistics, by allowing the loss function to \emph{depend on multiple data points} simultaneously.  In the language of property elicitation, this corresponds to loss functions such as $\loss(r,y_1,y_2)$ which judge the ``correctness'' of the report $r$ based on two (or more) observations $y_1$ and $y_2$.  Assuming these observations are drawn independently from the same distribution, this intuitively gives the loss function more power, and could potentially render previously non-elicitable statistics elicitable.  In fact, the variance is one such example: if $y_1$ and $y_2$ are both drawn i.i.d.~from $p$, it is easy to see that $\tfrac 1 2 (y_1-y_2)^2$ will be an unbiased estimator for the variance of $p$, hence $\loss(r,y_1,y_2) = (r-\tfrac 1 2 (y_1-y_2)^2)^2$ elicits the variance for the usual reason that squared error elicits expected values.
Examples of settings where such i.i.d.\ observations are readily obtained include: active learning, uncertainty quantification \& robust engineering design~\citep{beyer2007robust}, and replication of scientific experiments.

Beyond the variance, are there other non-elicitable statistics which we can elicit with multiple i.i.d.~observations?  Moreover, what is the tradeoff between the number of observations and the number of reports?  One would expect the elicitation complexity, in the usual number-of-reports sense, to drop as observations are added, but how fast is unclear.  Indeed, we will see several examples where the complexity drops dramatically, such a the $k$-norm of the distribution $p$.  In Section~\ref{sec:lower} we develop new techniques to prove complexity bounds using algebraic geometry, which show for example that the complexity of the $k$-norm drops from the support size of $p$ (minus 1) with 1 observation, to $1$ with $k$ observations.  We call the feasible (\# reports, \# observations) pairs the \emph{elicitation frontier}, for which the given statistic is elicitable, a concept we explore in Section~\ref{sec:elic-front}.

Finally, in Section~\ref{sec:regression} we apply multi-observation elicitation to regression.  Traditional elicitation complexity expresses a conditional statistic $\Gamma$ as a link of other statistics, but as we illustrate, situations can arise where these other statistics have a much more complicated relationship with the covariates than $\Gamma$ does.  We give an example where fitting a model to the conditional variance directly (using nearby data points as proxies for i.i.d.~observations) is much better than fitting separate models to the conditional first and second moments and combining these to obtain the variance.

\subsection{Related work}
Our work is inspired in part by~\cite{frongillo2015elicitation-1} which proposes a way to elicit the confidence (inverse of variance) of an agent's estimate of the bias of a coin by simply flipping it twice.
In our terminology, this follows from the fact that the variance is $(1,2)$-elicitable.
Multi-observation losses have been previously introduced to learn embeddings~\citep{hadsell2006dimensionality,schroff2015facenet:,ustinova2016learning}, though an explicit property/statistic is never discussed.

\section{Preliminaries}
\label{sec:preliminaries}

We are interested in a space $\O$ 
from which \emph{observations} $\o$ are drawn,
which will be a finite set unless otherwise specified.
We will denote by $\P\subseteq\Delta_\O$ a set of probability distributions of interest.
(Generally in this paper, $\P$ is simply the entire simplex.)
We refer to the set $\Delta_{\O^m}$ of all distributions on $m$ outcomes as the \emph{$m$-product space}.
To capture the assumption that we may collect $m\in \{1,2,\ldots\}$ observations which are each i.i.d.~from the same distribution $p\in\Delta_\O$, we will write $p^m \in \Delta_{\O^m}$ to denote their joint distribution, $p^m(\o_1,\ldots,\o_m) = \prod_i p(\o_i)$.
The set of all such distributions is denoted $\P^m = \{p^m : p \in \P\}\subseteq \Delta_{\O^m}$, which we will think of as a manifold in the $m$-product space.

With this notation in hand, we can define the central concepts in elicitation complexity in our context.
Properties include any typical statistic,\footnote{As defined, statistics like the median would not be included unless restrictions were placed on $\P$ for them to be single-valued (distributions in general may have multiple medians); we may instead extend our definition to include set-valued statistics, which would not substantially alter our results, and in fact we do lift this restriction in Section~\ref{sec:finite-properties}.} for instance, the mean when $\O \subseteq \reals$ is the property $\Gamma(p) = \sum_{\o} p(\o) \o$.
\begin{definition}[Property]
  A \emph{property} is a function $\Gamma:\P\to\R$, where $\R \subseteq \reals^k$ for some $k\geq 1$.
\end{definition}

Intuitively, properties represent the information desired about the data or underlying distribution.  $\R$ is sometimes called the \emph{report space}.
The central notion of property elicitation is the relationship between a loss function $\ell$ and the minimizer of its expected loss.  If this minimizer is a particular property $\Gamma$, we say $\ell$ \emph{elicits} $\Gamma$.  We simply extend this usual definition to allow for multiple observations in the expected loss.
\begin{definition}[Loss function, elicits]
  An \emph{$m$-observation loss function} is a function $\loss:\R\times\O^m\to\reals$, where $\loss(r,\o_1,\ldots,\o_m)$ is the loss for prediction $r\in\R$ scored against realized observations $\o_i\in\O$.
  We say $\loss$ \emph{(directly) elicits} a property $\Gamma:\P\to\R$ if for all $p\in\P$ we have $\{\Gamma(p)\} = \argmin_{r\in\R} \E_{(\o_1,\ldots,\o_m)\sim p^m}[\loss(r,\o_1,\ldots,\o_m)]$.
\end{definition}

It is useful to consider a property in terms of its \emph{level sets}, the set of distributions sharing the same particular value of the property.
For example, when the property is the mean of a distribution on $\{1,2,3,4\}$, both $p=\left(\frac{1}{2},0,0,\frac{1}{2}\right)$ and $p=\left(0,\frac{1}{2},\frac{1}{2},0\right)$ lie in the level set $\Gamma_{2.5}$.
\begin{definition}[Level set]
A \emph{level set} $\Gamma_r$ of a property $\Gamma:\P\to\R$ is, for $r \in \R$, the set of distributions with property $r$, i.e.  $\Gamma_r = \{p \in \P \; | \; \Gamma(p) = r\}$.
\end{definition}

An important technical condition on a property, and one which we will need for the notion of indirect elicitability, is that it be \emph{identifiable}, meaning that its level sets can be described by linear equalities.
\begin{definition}[Identifiable]
  A property $\Gamma:\P\to\R$, with $\R\subseteq\reals^k$, is \emph{identifiable with $m$ observations} if there exists some $V: \R \times \O^m\to\reals^k$ such that $\Gamma(p)=r \iff \E_{p^m}[V(r,\vec{\o})] = 0 \in \reals^k$, where $\vec{\o} = (\o_1,\ldots,\o_m)$ is drawn from $p^m$.
  We also say it is \emph{$m$-identifiable}.
\end{definition}

Identifiability is a geometric restriction on properties that is intuitively similar to continuity of the property (cf.~\citet{lambert2008eliciting, steinwart2014elicitation}).
Technically, observe that \emph{differentiable} loss functions generally elicit an identifiable property, as any local optimum should have $\sum_i \frac{\partial}{\partial r_i}\loss\left(r,\vec{\o}\right) = 0$, meaning that the gradient of $\loss$ itself gives an identification function.
Following \citet{frongillo2015vector-valued}, we will often assume that properties are identifiable.

Notice that any property can be ``indirectly'' elicited by using a proper scoring rule, which elicits the entire distribution, and then computing the property from the distribution.
But this requires a report of dimension $|\O|-1$, whereas to indirectly elicit the variance of $\o$, for example, requires just two reports, e.g. $r_1 = \E \o$ and $r_2 = \E \o^2$, along with a ``link function'' $\psi(\vec{r}) = r_2 - r_1^2$.
The question of \emph{elicitation complexity}, studied by~\cite{lambert2008eliciting} and~\cite{frongillo2015elicitation-2}, is how many dimensions $d$ are needed to indirectly elicit the property of interest $\Gamma$ via some elicitable $\hat\Gamma:\P\to\reals^d$; one hopes that $d$ is much smaller than $|\O|$.
Here we augment this question by another degree of freedom: how many dimensions $d$, \emph{and observations $m$}, are needed to indirectly elicit $\Gamma$?

\begin{definition}[$(d,m)$-elicitable]
  A property $\Gamma: \P\to\R$ is \emph{$(d,m)$-elicitable} if there exists a $d$-dimensional and identifiable property $\hat{\Gamma}: \P \to \hat\R$ where $\hat\R\subseteq \reals^d$, an $m$-observation loss function $\loss: \hat\R \times \O^m \to \reals$, and a ``link'' function $\psi: \hat\R \to\R$, such that
  \begin{center}\vspace*{-5pt}
    1. $\loss$ directly elicits $\hat{\Gamma}$, and~~~ 2.
      $\Gamma(p) = \psi \left( \hat{\Gamma}(p) \right)$.
  \end{center}\vspace*{-5pt}
  The \emph{elicitation frontier} of $\Gamma$ is the set of $(d,m)$ such that $\Gamma$ is $(d,m)$-elicitable, but neither $(d-1,m)$- nor $(d,m-1)$-elicitable.
\end{definition}
We may say that a property's ``report complexity'' is $d$ if $(d,1)$ lies on its frontier, and its ``observation complexity'' is $m$ if $(1,m)$ does.

\subsection{Illustrative example}

Recall our observation that the variance is not $(1,1)$-elicitable, and the ``traditional'' fix is to utilize $(2,1)$-elicitability: minimize a loss function over two dimensions (say first and second moments), mapping the result to the variance via a link function.
We observed instead that it is possible to utilize $(1,2)$-elicitability: minimize a loss function that takes two observations over a single scalar, the variance itself.
Can this tradeoff be more extreme?  In particular, are there cases where additional observations drastically decrease the report complexity?
Consider the 2-norm of a distribution: $\Gamma(p) = \|p\|_2 = \sqrt{\sum_{\o} p(\o)^2}$.  We show in Section \ref{sec:norms} that $\|p\|_2$ has report complexity $|\Y|-1$ (where $\Y$ is the outcome set) for 1 observation -- no single-observation loss function can do better than solving for the entire distribution.
However, recall that $\|p\|_2^2 = \sum_y p_y^2 = \Pr[\o_1 = \o_2]$ for two i.i.d. observations $\o_1,\o_2$, or in other words, $\|p\|_2^2 = \E_p \ones\{\o_1 = \o_2\}$.
The two-norm is actually elicitable with two observations and a single dimension using e.g. loss function $\loss(r,y_1,y_2) = (r - \ones\{y_1=y_2\})^2$, then simply computing $\|p\|_2 = \sqrt{r}$.
In other words, the two-norm's elicitation frontier on $\O$ consists of the points $(|\O|-1, 1)$ and $(1,2)$.

The goal for this paper is to investigate the (algebraic-)geometric reasons underpinning why a property might have low or high observation complexity, as well as providing general results and examples based on these ideas.
We next introduce the geometric foundations for this investigation.

\section{Geometric Fundamentals}
The most basic (yet powerful) lower bound in property elicitation says that elicitable properties' level sets must be convex sets \citep{lambert2008eliciting}.
Indeed, this is used to prove the variance is not (1,1)-elicitable; but the variance \emph{is} elicitable with two observations.
The geometry is not ``broken'' here, but merely lives in a higher-dimensional space.
When reasoning about eliciting a property $\Gamma:\P\to\reals$ using $m$ observations, it often useful to instead think of eliciting the property using \emph{a single random draw from a distribution on $m$-tuples of outcomes}.

\begin{remark}\label{R:m-elicit}
  Since $\mathcal P$ is 
  isomorphic to $\mathcal P^m$, a property $\Gamma:\mathcal P\to \R$ is directly elicitable with $m$ observations if and only if the induced property $\Gamma^m:\mathcal P^m\to \R$ is directly elicitable with $1$ observation.  
In particular, a sufficient condition for $(d,m)$-elicitability of $\Gamma$ is that there exists some $(d,1)$-elicitable $\Gamma': \Delta_{\O^m} \to \R$ that coincides with $\Gamma^m$ on $\P^m$.
One can elicit $\Gamma$ using the same loss that elicits $\Gamma'$, treating the $m$-tuple of observations as a single draw from the larger space.
\end{remark}

This gives us one initial way to demonstrate that a property is elicitable with $m$ observations.
For example, the loss function $\loss(r, a,b) = \left(r - \frac{1}{2}(a-b)^2\right)^2$ elicits the variance with two observations $a,b$, but if we consider distributions on all of $\O \times \O$, including non-i.i.d. distributions, it actually is still a valid loss function eliciting a property that coincides with the variance when $a,b$ are i.i.d.
To see this, just note that it still elicits an expectation: $\sum_{a,b} p'(a,b) \frac{1}{2}(a-b)^2$ where $p'$ is a distribution on $\reals^2$.

\ifHideDiag
\begin{wrapfigure}{R}{0.3\textwidth}
  \includegraphics[trim=5 30 109 42, clip, width=0.3\textwidth]{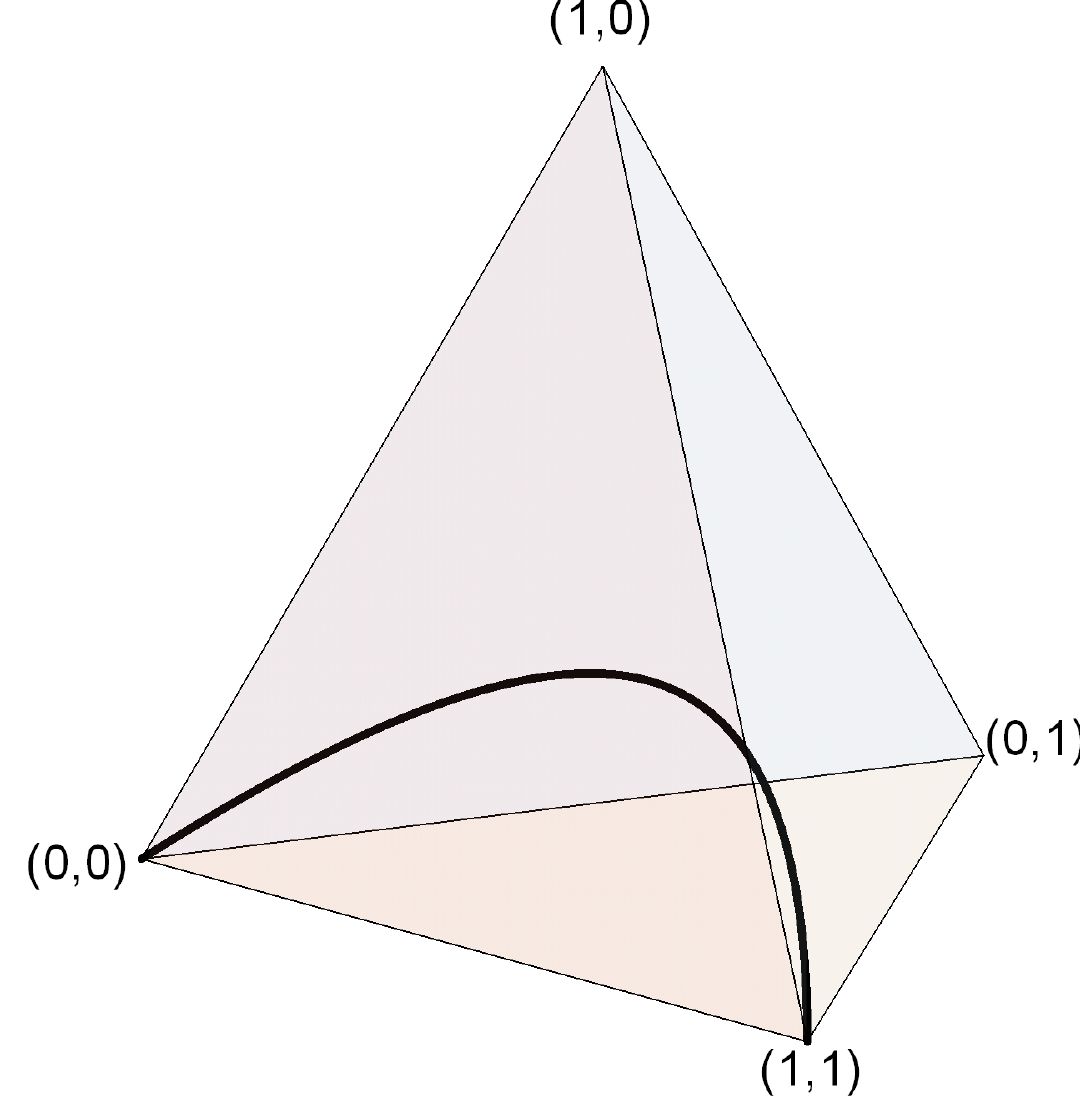}
  \vspace{-1.75em}
  \caption{The two outcome, two observation probability simplex $\Delta_{\O^2}$ where $\O=\{0,1\}$.  The arc is the space of i.i.d. distributions $(\Delta_\O)^2$.}
  \label{fig:2outcome_space}
  \vspace{-1em}
\end{wrapfigure}
\else
\fi

However, considering elicitation on the larger space $\Delta_{\Y^m}$ does not resolve the problem in either the necessary or sufficient directions.  First, $\P^m$ is not a convex set for $m > 1$, so conditions on the convexity of level sets do not naturally extend here.  An example of this is shown in Figure \ref{fig:2outcome_space}.  Second, coming up with an ``extended property'' may be difficult or non-obvious.
For example, it is not so clear whether the above loss function elicits anything \emph{natural} on $\Delta_{\O^2}$ (it is not the covariance, for instance, which is zero for i.i.d. distributions).
More fundamentally, it is not clear whether such extensions should generally exist.
(Proving or constructing a counterexample is an interesting open problem.)
In general, we hope to be able to accomplish much more by restricting to $\P^m$ because it is only a tiny $|\Y|$-dimensional manifold in a $|\Y|^m$-dimensional space.

A tighter sufficient condition is given by \citet{frongillo2014general}, which states that essentially all loss functions eliciting a property on any set, such as $\P^m$, also elicit some ``extension'' of that property on the convex hull of that set.
So while the higher-dimensional approach is helpful, it does not preclude reasoning about the space $\P^m$ as a manifold inside $\Delta_{\Y^m}$.

Most significantly, $\P^m$ is not a convex space, which makes lower bounds on elicitation complexity nontrivial as well.
However, the result of \citet{frongillo2014general} shows that it suffices to provide lower bounds for elicitation on the convex hull of $\P^m$, which we will denote $\conv(\P^m)$.  
Quite naturally then, we explore what leverage we can gain by reasoning about $\conv(\P^m)$.

\begin{theorem} \label{thm:direct_lower}
The property $\Gamma:\P\to\reals$ is not directly elicitable with $m$ observations if there exists $r_1, r_2 \in \Gamma(\P)$, $p_{1,1},\ldots,p_{1,k_1} \in \Gamma_{r_1}$, $p_{2,1},\ldots,p_{2,k_2} \in \Gamma_{r_2}$, $\lambda_{1,1},\ldots,\lambda_{1,k_1} \in [0,1]$ and $\lambda_{2,1},\ldots,\lambda_{2,k_2} \in [0,1]$ such that $r_1 \neq r_2$, $\sum_{i=1}^{k_1} \lambda_{1,i} = 1$, $\sum_{i=1}^{k_2} \lambda_{2,i} = 1$ and $$\sum_{i=1}^{k_1} \lambda_{1,i} p^m_{1,i} = \sum_{i=1}^{k_2} \lambda_{2,i} p^m_{2,i}.$$
\end{theorem}

\ifHideDiag
\begin{figure}
	\centering
\includegraphics[width=.5\textwidth]{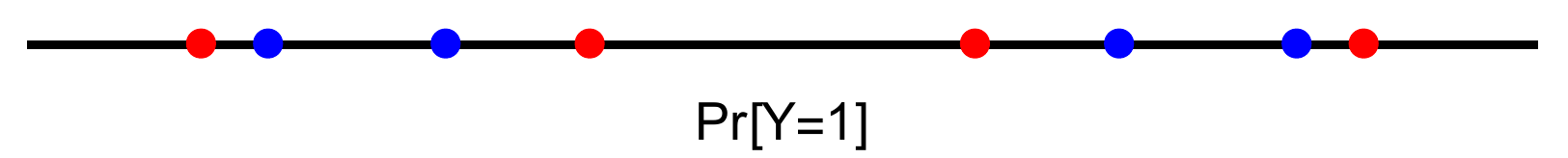}
	\includegraphics[width=.8\textwidth]{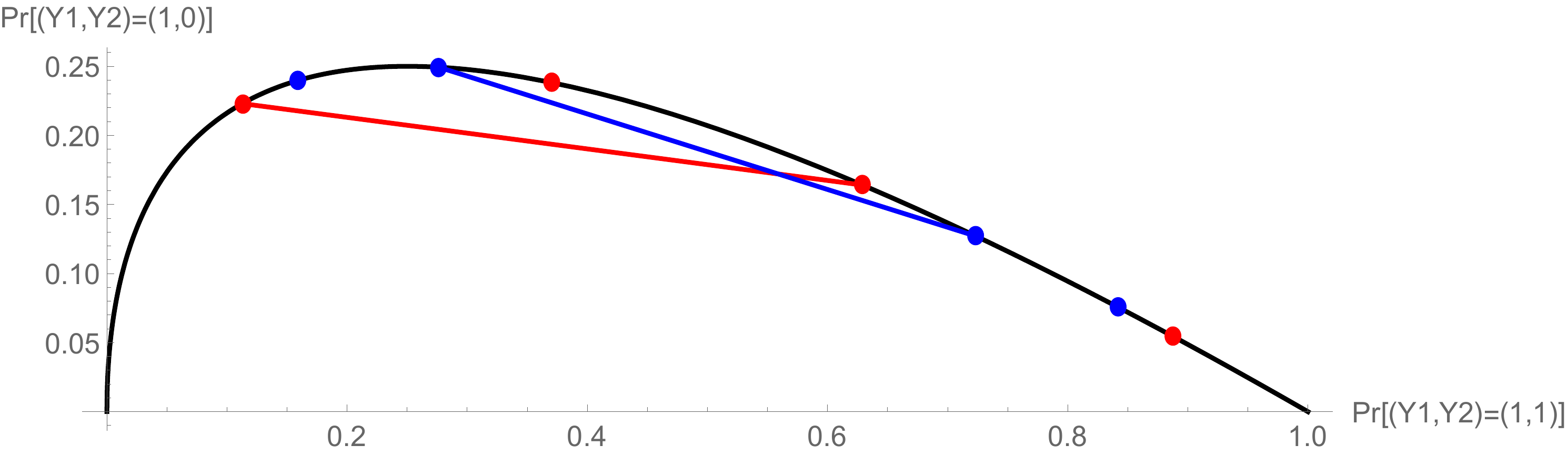}
	\caption{\textbf{Top:} The red dots and blue dots are each a level set of the fourth central moment of a Bernoulli random variable $Y \sim p$. These correspond to the distributions with fourth central moments .07 and .08 respectively. \textbf{Bottom:} The curve is $\Delta_{\O^2}$ projected into $\reals^2$, and the colored dots are the level sets of the example above projected into this space.  The lines demonstrate that there is a point in $\conv(\Delta_{\O^2})$ that can be written as a convex combination of either of the two level sets. }
	\label{fig:4th_moment}
  \vspace{-2em}
\end{figure}
\else
\fi

In other words, a property is not elicitable if there is a convex combination of one of its level sets in the $m$-product space that equals a convex combination of another one of its level sets in the $m$-product space.  

Theorem \ref{thm:direct_lower} allows us to prove for example that the fourth central moment is not directly elicitable with two observations.  Consider a Bernoulli random variable $Y \sim p$, then two of the level sets of the fourth central moment $\Gamma(p)=\E_{Y \sim p}[(Y-E_{Y \sim p}[Y])^4]$ are given in Figure \ref{fig:4th_moment}.  When we project these level sets into the 2-product space we can easily find a pair of points from each level set whose connecting lines intersects in $\conv(\Delta_{\O^m})$.  These lines are convex combinations of points in the same level set, so by Theorem \ref{thm:direct_lower} the lines' intersection implies that $\Gamma$ is not directly elicitable with two observations.

\subsection{Finite Properties}
\label{sec:finite-properties}
\emph{Finite properties} are those where $\R$, the range of $\Gamma$, is a finite set.
This corresponds to a ``multiple-choice question''~\citep{lambert2009eliciting}.
In this section, we must allow $\Gamma: \P \rightrightarrows \R$ to be a set-valued function, possibly assigning multiple possible correct reports to a single distribution; this is necessary for ``boundary'' cases, such as the mode of the uniform distribution on a finite set.
(Similarly, we cannot require identifiability.)
We have a finite set of outcomes $\Y$, the distributions considered are all $\P = \Delta_{\Y}$, and $\Gamma(p)$ must be nonempty.

We are interested in understanding which finite properties can be elicited with $m$ observations.
Previously, this question was studied for the case of one observation by \citet{lambert2011elicitation}, who characterized elicitable properties by the shape of their level sets: they are intersections of \emph{Voronoi diagrams} in $\reals^{|\O|}$ with the simplex $\Delta_{\Y}$.
In our setting, a Voronoi diagram is specified by a finite set of points $\{x_r : r \in \R\} \subseteq \reals^{\O^m}$, with each \emph{cell} $T_r = \{x : \|x - x_r\| \leq \|x - x_{r'}\| \forall r' \in \R\}$ consisting of those points in $\reals^{\O^m}$ closest in Euclidean distance to $x_r$.

Using the geometric constructions above, we can simply apply the main result of \citet{lambert2011elicitation} to finite properties in the $m$-product space; the result is a characterization of elicitable finite properties with $m$ observations.
\begin{corollary}
  A finite property $\Gamma: \Delta_{\O} \rightrightarrows \R$ is directly elicitable with $m$ samples if and only if there exists a Voronoi diagram in $\reals^{\O^m}$ with $\{x_r : r \in \R\}$ satisfying $\Gamma_r^m = T_r \cap \P^m$.
  Here $\Gamma_r^m = \{p^m \in \P^m : p \in \Gamma_r\}$.
\end{corollary}

\ifHideDiag
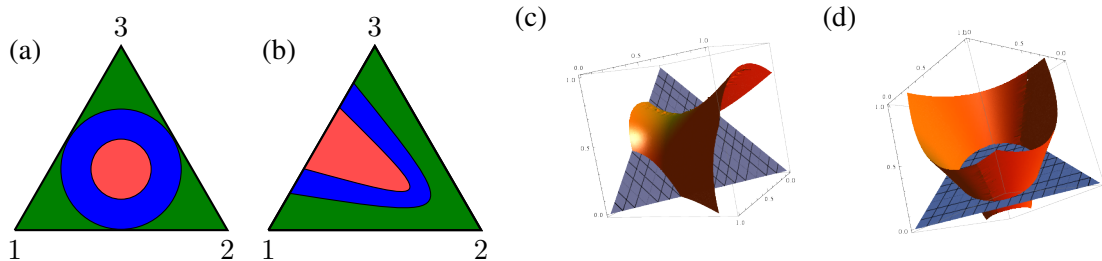
\begin{figure}
  \begin{tikzpicture} [scale=2, tdplot_main_coords]
    \coordinate (orig) at (0,0,0);
    \coordinate[label=below:$1\vphantom{i}$] (h) at (1,0,0);
    \coordinate[label=below:$2$] (m) at (0,1,0);
    \coordinate[label=above:$3$] (l) at (0,0,1);

    \draw[simplex,fill=darkgreen] (h) -- (m) -- (l) -- (h);
    \draw[fill=blue] (orig) circle [radius=4mm];
    \draw[fill=red!70!white] (orig) circle [radius=2mm];

    \placefiglabel{(a)}
  \end{tikzpicture}
  \begin{tikzpicture} [scale=2, tdplot_main_coords]
    \coordinate (orig) at (0,0,0);
    \coordinate[label=below:$1\vphantom{i}$] (h) at (1,0,0);
    \coordinate[label=below:$2$] (m) at (0,1,0);
    \coordinate[label=above:$3$] (l) at (0,0,1);

    \draw[simplex,fill=darkgreen] (h) -- (m) -- (l) -- (h);
    \draw[fill=blue] (1/5,0,4/5) .. controls (0,0.85,0) .. (4/5,0,1/5);
    \draw[fill=red!70!white] (1/3,0,2/3) .. controls (0,0.6,0) .. (2/3,0,1/3);
    \draw[simplex] (h) -- (m) -- (l) -- (h);

    \placefiglabel{(b)}
  \end{tikzpicture}
  \begin{tikzpicture}
    \node at (2,2) {\includegraphics[width=0.2\textwidth]{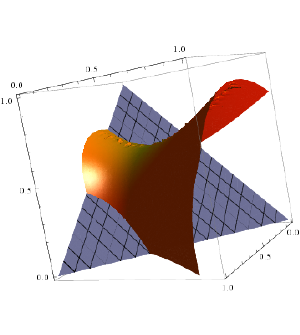}};

    \node at (0,3.5) {(c)};
  \end{tikzpicture}
  \begin{tikzpicture}
    \node at (2,2) {\includegraphics[width=0.2\textwidth]{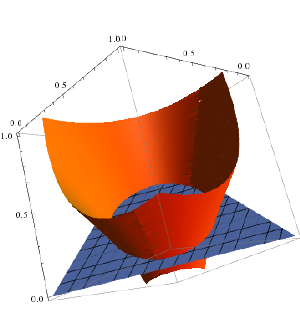}};

    \node at (0,3.5) {(d)};
  \end{tikzpicture}
  \caption{\textbf{Examples of finite properties on $\O = \{1,2,3\}$ elicitable with $2$ samples.}
    Pictured is the simplex on $3$ outcomes and properties $\Gamma: \Delta_{\{1,2,3\}} \to \{\text{red, green, blue}\}$.
    The agent reports a color, then is rewarded according to which outcome occurs.
    (a) The property of ``close'', ``intermediate'', and ``far'' from uniform, as measured by $2$-norm.
    (b) The property of ``high'', ``medium'', and ``low'' variance.
(c,d) The boundary between two cells, i.e.~a hyperplane in the $m$-product space ``projected'' down to $\reals^\O$ (in orange) and intersected with the simplex $\Delta_\O$ (in blue/gray); we show the boundary on all of $\reals^\O$ to visualize the quadratic surfaces which create these sections.}  
  \label{fig:finite-examples}
  \label{fig:finite-surface}
\end{figure}
\else
\fi

Multiple observations afford considerable flexibility in the level sets of such an elicitable $\Gamma$.  In particular, whereas before the cell boundaries between level sets were restricted to hyperplanes, with $m$ observations these boundaries can be defined by nearly arbitrary $m$-degree polynomials.  We illustrate this flexibility and visualize the cell boundaries in Figure~\ref{fig:finite-examples}.  In particular, we show that a classic negative example, where an agent is asked to report whether their belief has low or high variance, is easily elicited with two observations.

\section{Lower Bounds via Geometry} \label{sec:lower}

In this section we discuss  lower bounds on elicitation complexity. 
For technical reasons we will here require $\P$ to be a $C^{\infty}$ submanifold of $\Delta_{\O}$ with corners.
Our lower bounds will also generally require $\Gamma$ to be a $C^{\infty}$ function, in which case we call it a \emph{$C^{\infty}$ property}.

We begin in the first subsection by recalling the structure of the level sets of identifiable properties, and then introduce a  technique for obtaining  from this some lower bounds on elicitation complexity  via differential geometry.  In the next subsection we focus on polynomial properties, and explain some results that use algebraic geometry to obtain sharp bounds.

 \subsection{Preliminaries on identifiable properties}

We start by recalling a  general method, introduced in~\citet{frongillo2015elicitation-2},  for showing lower bounds on elicitation complexity:
 \emph{Given a property $\Gamma$, if one can show that no level set from any $\hat\Gamma$, which is $m$-identifiable and directly elicitable with $m$ observations, can be contained in a particular level set of $\Gamma$, then $\Gamma$ cannot be $(d,m)$-elicitable}.  This follows immediately from the definitions:  if $\Gamma$ is indirectly elicited via $\hat\Gamma$ and link $\psi$, so that $\Gamma = \psi \circ \hat\Gamma$, then we have the following relationship between the level sets of $\Gamma$ and $\hat{\Gamma}$:
\begin{equation}\label{E:Lev-Set-U}
 \Gamma_r = \bigcup_{\hat{r} : \psi(\hat{r}) = r} \hat{\Gamma}_{\hat{r}}.
 \end{equation}
 In other words, the level sets of $\Gamma$ are obtained by combining some of the level sets of $\hat{\Gamma}$.   For instance, if  $\psi$ is a bijection, then the level sets of $\Gamma$ and $\hat \Gamma$  are identical.  
This method was used successfully  in \citet{frongillo2015elicitation-2} to show lower bounds on the report complexity ($d$) of a property, with $m=1$.
In this section, we will use the same method to show lower bounds on observation complexity ($m$), with $d=1$.

Our main tool for obtaining these lower bounds will be that the level sets of any directly $m$-observation-elicitable, identifiable $\hat{\Gamma}$ have a specific structure, namely, such a level set is  the zero set of a polynomial of degree at most $m$:
\begin{fact} \label{fact:ident-poly}
  If a property $\hat{\Gamma}(p)$ is $m$-identifiable, then each  level set of $\hat \Gamma$ is  the set of zeros of a  polynomial in $p$  of degree at most $m$.
\end{fact}
\begin{proof}
  The condition $\E_p V(r,\vec{\o}) = 0$ is $\sum_{\o_1,\ldots,\o_m} p(\o_1)\cdots p(\o_m) V(r,\o_1,\ldots,\o_m) = 0$ .
\end{proof}
Combined with the equality \eqref{E:Lev-Set-U} above, Fact~\ref{fact:ident-poly} tells us that the level sets of indirectly elicitable $\Gamma$ are unions of zero sets of polynomials.  As we are focusing on the $d=1$ case, however, both $\Gamma$ and $\hat\Gamma$ are real-valued functions, so with enough regularity, their level sets should coincide.
Before making a precise statement, we introduce the following definition:

\begin{definition}[$C^\infty$ $(d,m)$-elicitable]
\label{D:Cinfdm}
  We say that  a $C^\infty$
 property $\Gamma:\mathcal P\to \mathbb R^{d'}$ is $C^\infty$ $(d,m)$-elicitable if  in the definition of $(d,m)$-elicitable,  $\hat \Gamma$   can be taken to be $C^\infty$ and $\psi$ can be taken to be $C^\infty$ in an open neighborhood of the image of $\hat \Gamma$. 
  \end{definition}

\begin{corollary} \label{cor:connected-zeros}
Suppose that a $C^\infty$
 property $\Gamma:\mathcal P\to \mathbb R$ is $C^\infty$ $(1,m)$-elicitable.  
Let $r\in \mathbb R$, let $Z\subseteq \Gamma^{-1}(r)$ be a connected component of the level set, and assume that $Z$ admits a point that is not a critical point of $\Gamma$; i.e., there is a point $p\in Z$ such that the differential of $\Gamma$ at $p$ is nonzero.   
Then $Z$ is a connected component of the set of zeros of a  polynomial of degree at most   $m$.  Moreover, if $\Gamma^{-1}(r)$ is connected, then $\Gamma^{-1}(r)$ is the zero set of a polynomial of degree at most $m$.  
\end{corollary}

\begin{proof}  Let $\widehat \Gamma$ and $\psi$ be as in Definition \ref{D:Cinfdm}.  
We have a commutative diagram:
$$
\xymatrix{
\mathcal P \ar[r]^{\widehat\Gamma} \ar[d]_{ \Gamma}&\mathbb R \ar[ld]^\psi\\
\mathbb R&
}
$$
Since $Z$ is connected, we have that $\widehat {\Gamma}(Z)\subseteq \mathbb R$ is connected, and is therefore an interval (see e.g., \cite{browder96},  Theorem 6.76, 6.77, p.148). 
The claim is that this interval is a point. Indeed, assume the opposite.  Then
 since $\psi$ is by definition constant on the interval  $\widehat \Gamma(Z)$,  we would have that the differential $D\psi $ vanishes at each point of  of $\widehat \Gamma(Z)$.  Then  since $D\Gamma =D\psi \circ D\widehat \Gamma$ we would have that $D\Gamma$ vanishes at every point of $Z$.  But this would  contradict our assumption.  Thus $\widehat\Gamma(Z)$ is a point.  
 
It then follows from Fact \ref{fact:ident-poly} that   $\widehat \Gamma^{-1}( \widehat\Gamma(Z))$ is the zero set of a polynomial of degree at most $m$.
We now use the inclusions 
$$
Z\subseteq  \widehat \Gamma^{-1}( \widehat\Gamma(Z))\subseteq \Gamma^{-1}(r).
$$ 
  By virtue of  the inclusion on the right, every  connected component of $\widehat \Gamma^{-1}( \widehat\Gamma(Z))$ is contained in a connected component of $\Gamma^{-1}(r)$. This proves the first assertion of the lemma.  The last assertion of the lemma also follows from these inclusions, since in that case one is assuming   $Z=\Gamma^{-1}(r)$.
\end{proof}

\begin{remark}
For concreteness, we summarize the  contrapositive of Corollary \ref{cor:connected-zeros} in the way in which we will use it in examples:  Suppose that $\Gamma:\mathcal P\to \mathbb R$ is a $C^\infty$ property, and there exists an $r\in\mathbb R$ such that the level set  $\Gamma^{-1}(r)$ is connected, and contains a  point $P\in \Gamma^{-1}(r)$ that is not a critical point for $\Gamma$.  Then if $\Gamma^{-1}(r)$ is not the zero locus of a degree $m$ polynomial in $p(\o_1),\ldots,p(\o_m)$, then $\Gamma$ is not $C^\infty$ $(1,m)$-elicitable.
\end{remark}

As a consequence of Corollary \ref{cor:connected-zeros}, we can immediately show the existence of $C^\infty$ properties with \emph{infinite} observation complexity; i.e., properties that are not $C^\infty$ $(1,m)$ elicitable for any $m$.  
The proof gives such an example for $|\O|=3$, a surprising result given that all properties have report complexity $|\O|-1=2$; i.e., all of the  $C^\infty$ properties are  $C^\infty$ $(2,1)$-elicitable.  Note  that  if $|\O|=2$, then all $C^\infty$ properties are $C^\infty$  $(1,1)$-elicitable.  

\begin{proposition} \label{prop:sine}
  There are $C^\infty$ properties that are not $C^\infty$ $(1,m)$-elicitable for any finite $m$.
\end{proposition}
\begin{proof}
  Take $\Y = \{1,2,3\}$, $\P=\Delta_\Y^\circ=\{p\in\Delta_\Y: p(y) > 0 \; \forall y\in\Y\}$, and $\Gamma(p) = p_1 - (1/2)\sin (1/p_2)$.
  It is immediate that $\Gamma$ has no critical points.  
  Here the level sets $\Gamma_r$ satisfy $r = p_1 - (1/2)\sin (1/p_2)$, in other words, satisfy the equation $p_1 = (1/2)\sin (1/p_2) + r$.
  For $p_2$ sufficiently small, the level set $\Gamma_0 = \{ p \in \Delta_3 : p_1 = (1/2)\sin (1/p_2) \}$ is simply the graph of $(1/2)\sin (1/x)$, which intersects the line $p_1=0$ infinitely many times, and hence by the Fundamental Theorem of Algebra is not the zero set of \emph{any} polynomial.
  Corollary~\ref{cor:connected-zeros} now implies that $\Gamma$ is not $(1,m)$-elicitable for any $m$.
\end{proof}

\subsection{Polynomial properties and lower bounds using algebraic geometry}
We now describe some lower bounds for elicitation complexity of polynomial properties.
The motivation for these lower bounds  is the  intuition that, in general, a polynomial property $\Gamma:\mathcal P\to \mathbb R$ of degree $k$ should not be $C^\infty$ $(1,m)$-identifiable for any $m<k$, since the zero set of a degree $k$ polynomial should not be the zero set of a degree $m$ polynomial when $m<k$.
This statement can of course fail in special cases (e.g., Example \ref{E:Poly-Counter} below).   Indeed,  there are some subtleties  regarding  zero sets of polynomials in Euclidean open sets, considered in Appendix \ref{sec:algebra-overview}, that  must be addressed to draw such a conclusion.  Nevertheless, for a general polynomial property  this expectation holds  (see Remark \ref{remark:most-polynomials} for a precise definition of generality), and in the appendix  we provide some elementary techniques for confirming this expectation in particular  examples (see Corollary \ref{C:f_c}).    For instance, we show (Example \ref{E:f_c}):

\begin{corollary} \label{cor:k-norm-obs} If $|\O|\ge 3$, then for any natural number $k$, the $k$-norm of a distribution, $\Gamma(p) = (\sum_{\o} p(\o)^k)^{1/k}$, is not $C^\infty$ $(1,k-1)$-elicitable.
\end{corollary}

\begin{example}\label{E:Poly-Counter}
In contrast to the case considered in Corollary \ref{cor:k-norm-obs}, we emphasize that there are polynomial properties $\Gamma:\mathcal P\to \mathbb R$ of degree $k$ that are $C^\infty$ $(1,m)$-elicitable for some $m<k$.  For instance, take $\hat \Gamma:\mathcal P\to \mathbb R$ to be any polynomial property of degree $m>0$,  let $\psi:\mathbb R\to \mathbb R$ be any polynomial function of degree  $m'>1$, and set $\Gamma =\psi \circ \hat \Gamma$.  Then $\Gamma$ is of degree $k=mm'>m$, but $\Gamma$ is $C^\infty$ $(1,m)$-elicitable, by Lemma \ref{L:Poly-Elic}.
\end{example}

\section{Examples and Elicitation Frontiers} 
\label{sec:elic-front}

We now combine our complexity lower bounds with upper bounds to make progress toward determining the elicitation frontiers of some potential properties of interest.  
See Figure~\ref{fig:frontiers} for a depiction of some of the elicitation frontiers described.
We begin with some general, straightforward, but versatile upper bounds.

\begin{lemma}
\label{lem:sum-prod}
For all $1 \leq i \leq n,\, 1 \leq j \leq m$, let $f_{ij}:\Y\to\reals$ be an arbitrary function such that $\E_p[f_{ij}(Y)]$ exists for all $p\in\P$.  Then
$\Gamma(p) = \sum_{i=1}^n \prod_{j=1}^m \E_p[f_{ij}(Y)]$ is $(1,m)$-elicitable.
\end{lemma}

\begin{proof}
Using $Y_1,\ldots,Y_m$ which are i.i.d. from $p$, then $\{f_{i1}(Y_1), \ldots f_{im}(Y_m) \}$ will be independent for all $i$.  Using properties of expectations (linearity and independence), we have
\begin{equation} 
  \sum_{i=1}^n \prod_{j=1}^m \E[f_{ij}(Y)]
  = \sum_{i=1}^n \prod_{j=1}^m \E[f_{ij}(Y_j)]
  = \sum_{i=1}^n \E\left[\prod_{j=1}^mf_{ij}(Y_j)\right]
  = \E\left[\sum_{i=1}^n \prod_{j=1}^mf_{ij}(Y_j)\right] 
\end{equation}
Now we see that using squared loss (or any loss for the mean) one can leverage these $m$ samples to elicit the desired sum of products, e.g.
$\loss(r,y_1,\ldots,y_m) = \left(r-\sum_{i=1}^n \prod_{j=1}^mf_{ij}(y_j)\right)^2$.
\end{proof}
The proof of Lemma~\ref{lem:sum-prod} simply constructs an unbiased estimator of the property of interest and elicits the mean of the estimator via squared error.  By a very natural extension, this technique also applies to ratios of expectations, as they are elicitable~\citep{gneiting2011making}: construct \emph{two} unbiased estimators, and elicit the ratio of their means.  We will give two instances of such ratios in the next subsection.

The following result establishes an upper bound that by now may seem natural: Under some conditions, a property that is itself an $m$-degree polynomial in $p$ is $(1,m)$-elicitable.

\begin{lemma}\label{L:Poly-Elic}
Suppose that $\Gamma:\mathcal P\to \mathbb R$ is a property such that $\Gamma=\psi\circ \Gamma'$ where $\Gamma':\mathcal P\to \mathbb R$ is  polynomial of degree $m$, and $\psi:\mathbb R\to \mathbb R$ is a  function that is $C^\infty$ on an open neighborhood of the  image of $\Gamma'$.    Then $\Gamma'$ is directly   $(1,m)$-elicitable, and $\Gamma$ is $C^\infty$ $(1,m)$-elicitable.  
\end{lemma}

\begin{proof}  It is enough to show $\Gamma'$ is directly  $(1,m)$-elicitable.  
This follows immediately from Lemma \ref{lem:sum-prod}.  Indeed, it is clear from the lemma that it is enough to show the result for monic monomials.  For this one takes the $f_{ij}$ in Lemma  \ref{lem:sum-prod} to  be characteristic functions $\ones_{\o}$ for $\o\in \O$.    
\end{proof}

\subsection{Ratios of expectations: index of dispersion and Sharpe ratio}
The \emph{index of dispersion} of a random variable $Y$ with positive mean is defined to be $\Var(Y)/\E[Y]$ \citep{cox1966statistical}.  The \emph{Sharpe ratio} of a random variable $Y$, which is a commonly-used measure of the risk-adjusted return of an investment, is defined similarly as $\E[Y]/\sqrt{\Var(Y)}$ \citep{sharpe1966mutual}.  Both the index of dispersion and the \emph{square} of the Sharpe ratio are $(1,2)$-elicitable by the above discussion: $\Var(Y) = \E_p[\tfrac 1 2 (Y_1-Y_2)^2]$, $\E_p[Y] = \E_p[Y_1]$, and $\E_p[Y]^2 = \E_p[Y_1Y_2]$, so any ratios of these terms is $(1,2)$-elicitable.  (The link function for the Sharpe ratio is thus the square root.)  For example, the index of dispersion is elicited by the loss $\loss(r,y_1,y_2) = r(y_1 - y_2)^2 - r^2 y_1$.

To finish describing the elicitation frontiers for these properties, we note that neither is $(1,1)$-elicitable as the level sets are not convex, but both are $(2,1)$-elicitable as we now show. For the index of dispersion, we can take $r_1 = E[Y]$ and $r_2 = E[Y^2]$, both elicitable as means, and then compute the property by $(r_2-r_1^2)/r_1$.  Similarly, for the same $r_1,r_2$, the Sharpe ratio can be written as $r_1/\sqrt{r_2-r_1^2}$.

\subsection{Norms of distributions}
\label{sec:norms}
As we have previously discussed, the $2$-norm is $(1, 2)$ elicitable.  For general $k$, the $k$-norm is $(1, k)$ elicitable with the following loss function
$\loss(r,y_1,\ldots,y_k) = (r - \ones\{y_1=\ldots=y_k\})^2$.  (This case also follows from Lemma~\ref{L:Poly-Elic}.)
This is a tight bound on the observation complexity, as we proved in Corollary \ref{cor:k-norm-obs} that the $k$-norm is not $(1,k-1)$ elicitable.  
As it turns out, the report complexity of the $k$-norm is $|\Y|-1$, meaning it is as hard to elicit with one observation as the entire distribution.  This follows from Theorem 2 of~\cite{frongillo2015elicitation-2}, specifically Section 4.2, as $\|p\|_k$ is a convex function of $p$.  An interesting open question, and one that will require additional algebraic tools, is the $k$-norm's elicitation frontier when we allow multiple dimensions and multiple observations.

\begin{corollary} \label{cor:k-norm-frontier}
  For $|\O| \geq 3$, the elicitation frontier of the $k$-norm contains $(|\O|-1,1)$ and $(1,k)$.
\end{corollary}

\subsection{Central Moments}

The \(n^{th}\) central moment $\mu_n$ of a random variable $Y$ is defined as
\begin{equation}\label{eq:central-moment-def}
   \mu_n =  \E[(Y - \E[Y])^n ] =  \sum_{i=0}^n (-1)^i \dbinom{n}{i} \E[Y]^i  \cdot \E[Y^{n-i}] ~,
\end{equation}
which we see is $(n,1)$-elicitable by simply eliciting $E[Y^i]$ for all $i \in \{1,\ldots,n\}$ and then combining the results.
As we will show, $\mu_n$ is also $(1,n)$-elicitable, and moreover, we can achieve other dimension-observation tradeoffs in between, such as $(\lfloor \sqrt{n} \rfloor + 1,\lceil \sqrt{n} \rceil)$.
The key idea is to partition the binomial sum~\eqref{eq:central-moment-def} into $k$ partial sums and factor out the highest power of \(\E[Y] \) from each, such that the \(j^{th}\) partial sum can be written as
\begin{equation} \label{eq:central-moment-partial-sum}
    \E[Y]^{\frac{j\cdot n}{k}} \sum_{i=0} ^{\frac{n}{k}-1}  (-1)^i \dbinom{{n}}{\frac{j \cdot n}{k} + i} \cdot \E[Y]^i \cdot \E\left[Y^{\frac{(j+1) \cdot n}{k} -1 -i}\right]~.
\end{equation} 
Doing so gives the following result.

\begin{theorem}
The  \(n^{th}\) central moment is \((k+1, \ceil[\big]{n/k}) \)- elicitable; \(0 < k \leq n \)
\end{theorem}
\begin{proof}
Consider the partial sum~\eqref{eq:central-moment-partial-sum} without the $\E[Y]^{j\cdot n/k}$ factor; by Lemma \ref{lem:sum-prod} each such factored sum is $(1,\ceil[\big]{n/k})$-elicitable, as the maximum number of terms in any product is \(\ceil[\big]{n/k}\). Since we have $k$ such factored sums, and need to additionally elicit the mean \(\E[Y]\) to compute their factors, the entire sum can be elicited using \(\ceil[\big]{n/k}\) observations and $k+1$ dimensions.
\end{proof}

When \( k = 0\), we can do much better than $m=\infty$: by Lemma~\ref{lem:sum-prod}, as the maximum number of terms in any product of~\eqref{eq:central-moment-def} is $n$, the term \((\E[Y])^n \), we have than $\mu_n$ is $(1,n)$-elicitable.  For lower bounds, little is known beyond $\mu_n$ not being $(1,1)$-elicitable~\citep{frongillo2015elicitation}.

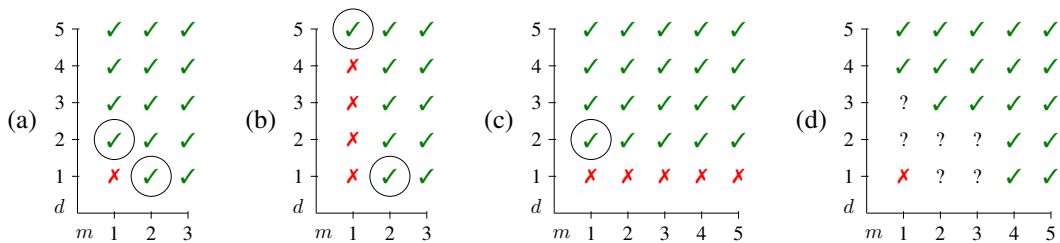
\begin{figure}[!hb]\centering
\scalebox{0.7}{\begin{tikzpicture}[y=.7cm, x=.7cm]
	\draw (0,0) -- coordinate (x axis mid) (3,0);
    	\draw (0,0) -- coordinate (y axis mid) (0,5);
    	\foreach \x in {1,...,3}
     		\draw (\x,1pt) -- (\x,-3pt)
			node[anchor=north] {\x};
        \node[anchor=north] at (0.2,-5pt) {$m$};
    	\foreach \y in {1,...,5}
     		\draw (1pt,\y) -- (-3pt,\y) 
     			node[anchor=east] {\y}; 
        \node[anchor=east] at (-3pt,0.2) {$d$};
        \node at (1,1) {\xmark};
        \node[draw,circle] at (2,1) {\cmark};
        \node at (3,1) {\cmark};
        \node[draw,circle] at (1,2) {\cmark};
        \node at (2,2) {\cmark};
        \node at (3,2) {\cmark};
        \node at (1,3) {\cmark};
        \node at (2,3) {\cmark};
        \node at (3,3) {\cmark};
        \node at (1,4) {\cmark};
        \node at (2,4) {\cmark};
        \node at (3,4) {\cmark};
        \node at (1,5) {\cmark};
        \node at (2,5) {\cmark};
        \node at (3,5) {\cmark};
        \node at (-1.5,2.5) {\Large (a)};
  \end{tikzpicture}}
\quad
\scalebox{0.7}{\begin{tikzpicture}[y=.7cm, x=.7cm]
	\draw (0,0) -- coordinate (x axis mid) (3,0);
    	\draw (0,0) -- coordinate (y axis mid) (0,5);
    	\foreach \x in {1,...,3}
     		\draw (\x,1pt) -- (\x,-3pt)
			node[anchor=north] {\x};
        \node[anchor=north] at (0.2,-5pt) {$m$};
    	\foreach \y in {1,...,5}
     		\draw (1pt,\y) -- (-3pt,\y) 
     			node[anchor=east] {\y}; 
        \node[anchor=east] at (-3pt,0.2) {$d$};
        \node at (1,1) {\xmark};
        \node[draw,circle] at (2,1) {\cmark};
        \node at (3,1) {\cmark};
        \node at (1,2) {\xmark};
        \node at (2,2) {\cmark};
        \node at (3,2) {\cmark};
        \node at (1,3) {\xmark};
        \node at (2,3) {\cmark};
        \node at (3,3) {\cmark};
        \node at (1,4) {\xmark};
        \node at (2,4) {\cmark};
        \node at (3,4) {\cmark};
        \node[draw,circle] at (1,5) {\cmark};
        \node at (2,5) {\cmark};
        \node at (3,5) {\cmark};
        \node at (-1.5,2.5) {\Large (b)};
  \end{tikzpicture}}
\quad
\scalebox{0.7}{\begin{tikzpicture}[y=.7cm, x=.7cm]
	\draw (0,0) -- coordinate (x axis mid) (5,0);
    	\draw (0,0) -- coordinate (y axis mid) (0,5);
    	\foreach \x in {1,...,5}
     		\draw (\x,1pt) -- (\x,-3pt)
			node[anchor=north] {\x};
        \node[anchor=north] at (0.2,-5pt) {$m$};
    	\foreach \y in {1,...,5}
     		\draw (1pt,\y) -- (-3pt,\y) 
     			node[anchor=east] {\y}; 
        \node[anchor=east] at (-3pt,0.2) {$d$};
        \node at (1,1) {\xmark};
        \node at (2,1) {\xmark};
        \node at (3,1) {\xmark};
<<<<<<< Updated upstream
        \node at (4,1) {\xmark};
        \node at (5,1) {\xmark};
        \node[draw,circle] at (1,2) {\cmark};
        \node at (2,2) {\cmark};
        \node at (3,2) {\cmark};
        \node at (4,2) {\cmark};
        \node at (5,2) {\cmark};
        \node at (1,3) {\cmark};
        \node at (2,3) {\cmark};
        \node at (3,3) {\cmark};
        \node at (4,3) {\cmark};
        \node at (5,3) {\cmark};
        \node at (1,4) {\cmark};
        \node at (2,4) {\cmark};
        \node at (3,4) {\cmark};
        \node at (4,4) {\cmark};
        \node at (5,4) {\cmark};
        \node at (1,5) {\cmark};
        \node at (2,5) {\cmark};
        \node at (3,5) {\cmark};
        \node at (4,5) {\cmark};
        \node at (5,5) {\cmark};
        \node at (-1.5,2.5) {\Large (c)};
  \end{tikzpicture}}
\quad
\scalebox{0.7}{\begin{tikzpicture}[y=.7cm, x=.7cm]
	\draw (0,0) -- coordinate (x axis mid) (5,0);
    	\draw (0,0) -- coordinate (y axis mid) (0,5);
    	\foreach \x in {1,...,5}
     		\draw (\x,1pt) -- (\x,-3pt)
			node[anchor=north] {\x};
        \node[anchor=north] at (0.2,-5pt) {$m$};
    	\foreach \y in {1,...,5}
     		\draw (1pt,\y) -- (-3pt,\y) 
     			node[anchor=east] {\y}; 
        \node[anchor=east] at (-3pt,0.2) {$d$};
        \node at (1,1) {\xmark};
        \node at (2,1) {?};
        \node at (3,1) {?};
        \node at (4,1) {\cmark};
        \node at (5,1) {\cmark};
        \node at (1,2) {?};
        \node at (2,2) {?};
        \node at (3,2) {?};
        \node at (4,2) {\cmark};
        \node at (5,2) {\cmark};
        \node at (1,3) {?};
        \node at (2,3) {\cmark};
        \node at (3,3) {\cmark};
        \node at (4,3) {\cmark};
        \node at (5,3) {\cmark};
        \node at (1,4) {\cmark};
        \node at (2,4) {\cmark};
        \node at (3,4) {\cmark};
        \node at (4,4) {\cmark};
        \node at (5,4) {\cmark};
        \node at (1,5) {\cmark};
        \node at (2,5) {\cmark};
        \node at (3,5) {\cmark};
        \node at (4,5) {\cmark};
        \node at (5,5) {\cmark};
        \node at (-1.5,2.5) {\Large (d)};
  \end{tikzpicture}}
  \caption{The elicitation frontiers for various properties: (a) the variance, Sharpe ratio, and index of dispersion; (b) the 2-norm when $|\O|=6$, with respect to $C_\infty$ properties;
   (c) $\Gamma(p) = p_1-(1/2)\sin(p_2\pi)$ from Proposition~\ref{prop:sine}; (d) the 4th central moment, which is not fully known.}
\label{fig:frontiers}
\end{figure}

\section{Multi-Observation Regression}
\label{sec:regression}
One of the earliest problems in modern statistics was the estimation of biodiversity in a geographic region~\citep{fisher1943relation}.
One scalar measure of diversity of a distribution is the (inverse of the) $2$-norm, which we will take here as an example.\footnote{A similar intuition will hold for most if not all elicitable measures of diversity.}
Consider a dataset of species samples: pairs $(x,\o)$ where $x$ gives the features of the geographic region and $\o$ is a categorical giving the species to which this sample belongs.
Suppose we wish to regress the diversity of species against geographic features such as climate.
The single-observation approach would require a surrogate loss function $\loss(\hat{f}(x), \o)$ and a link $f(x) = \psi(\hat{f}(x))$.
We claim that any single-observation loss function $\loss(f(x), \o)$ is poorly suited for this task.
For the $2$-norm, lower bounds on report complexity show that the best possible approach has dimensionality $\hat{f}: x \to \reals^{d-1}$ where $d$ is the number of unique species in the dataset (which may have a very long tail).
So this approach requires, in essence, fitting $\hat{f}$ to the entire distribution over species as a function of geographic region, a task of immense idiosyncrasy and complexity compared to the end goal of e.g. estimating a scalar measure of diversity as a function of rainfall level.

On the other hand, a two-observation loss function $\loss(f(x), \o_1,\o_2)$ can be used to directly learn an $f$ estimating the desired diversity measure, e.g. $2$-norm, as a function of geographic features.
One can then use empirical risk minimization to directly learn relationships between, e.g. rainfall level and this measure of species diversity.

Multi-observation regression does introduce an additional challenge, however: risk in this context is naturally defined as $\E_{x,\vec{\o}} \loss(f(x), \vec{\o})$ where $\vec{\o} = (\o_1,\ldots,\o_m)$ is a set of observations drawn i.i.d. conditioned on $x$, but our data points are of the form $(x,\o)$.
If e.g. $x$ comes from a continuous space, we may not have \emph{any} sets of $m$ samples $\o_1,\ldots,\o_m$ belonging to the same $x$.
One natural setting where this poses no concern is in active learning where we may choose to re-draw the label for a given $x$.
In a more standard regression framework, we propose to leverage the intuition that the distribution of $\o$ conditioned on $x$ generally changes gradually as a function of $x$.\footnote{Phrased differently, at least it seems reasonable to parameterize the rate of change and expect learning bounds to depend on this parameter.}
Pragmatically, with dense enough data points, we can simply group together nearby $x$ values and ``merge'' them into a data point of the form $(\bar{x}, \o_1,\ldots,\o_m)$ where $\bar{x}$ is an average and the $\o_i$ are drawn independently and \emph{approximately} identically from \emph{approximately} the distribution of $\O$ conditioned on $\bar{x}$.
For this paper, we demonstrate the idea in simulations below and give a basic proof-of-concept theoretical result in Appendix \ref{app:regression}, leaving a more thorough investigation to future work.

In general, the cases where the multi-observation approach can be useful are those where the property of interest is believed to follow a simple functional form, but the conditional statistics given by the indirect elicitation approach are expected to follow unknown or complicated trends as a function of features.
For another example, one could imagine learning the noise (e.g. variance) of a medical test, e.g. white blood cell count, as a function of patient features, in order to improve the test.
The indirect elicitation approach suggests first fitting a model for estimating the mean of the test's outcome as a function of patient data, then fitting the expected square of the statistic, and then computing an estimate for the variance by combining them.
In general, these prediction problems may be highly complex and nonlinear even when the \emph{noise} in the test might follow some simple linear relationship with e.g. height or age.
The multi-observation approach allows direct regression of the noise versus features.
Formally, we show a basic extension of classic risk guarantees in Appendix \ref{app:regression}, under the assumption that $x$ is distributed uniformly on $[0,1]$ and a closeness condition on the conditional distribution of $Y$ given $X$.

\subsection{Simulation}

Here we describe some simulations run as a proof of concept of multi-observation regression.  Our data points are of the form $(x, y) \in \reals \times \reals$ where $x$ is drawn uniformly at random from the interval $[0,1]$.  Given $x$, $y = a \sin(4 \pi x) + Z$, where $a$ is a constant and $Z \sim N(0,1)$ is drawn independently for each sample, we wish to learn $\Var(Y | X)$.

Our multi-observation loss function here is $\loss(f(x), y_1, y_2) = (f(x) - \frac{1}{2}(y_1-y_2)^2)^2$.  We approximate $(x, y_1, y_2)$ samples by sorting the $(x_i,y_i)$ pairs by $x_i$, and making samples of the form $(\frac{1}{2}(x_i+x_{i+1}),y_i,y_{i+1})$.  We compare to the single observation approach, in which we estimate $\E[Y | X]$ and $\E[Y^2 | X]$ and then combine them to estimate $\Var(Y | X)$.  

The point of these simulations is to demonstrate that multi-observation regression can greatly outperform single observation regression in the case when the function is in a known concept class, and the statistics needed to indirectly elicit it with a single observation are not in a known concept class.  As such, our multi-observation regression fits a linear function to $\Var(Y | X)$, and our single observation regression fits linear functions to $\E[Y | X]$ and $\E[Y^2 | X]$.  The true $\Var(Y | X) = 1$ is indeed a linear function, while the true moment functions $\E[Y | X\!=\!x] = a \sin(x)$ and $\E[Y^2 | X\!=\!x] = a^2 \sin^2(x) + 1$ are very far from linear.

\ifHideDiag
\begin{figure}
	\centering
		\includegraphics[width=.49\textwidth]{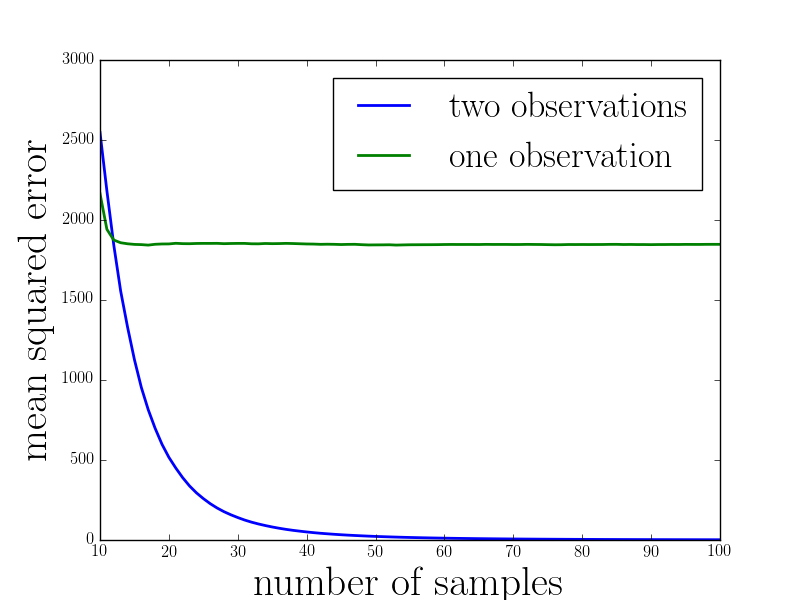}
		\includegraphics[width=.49\textwidth]{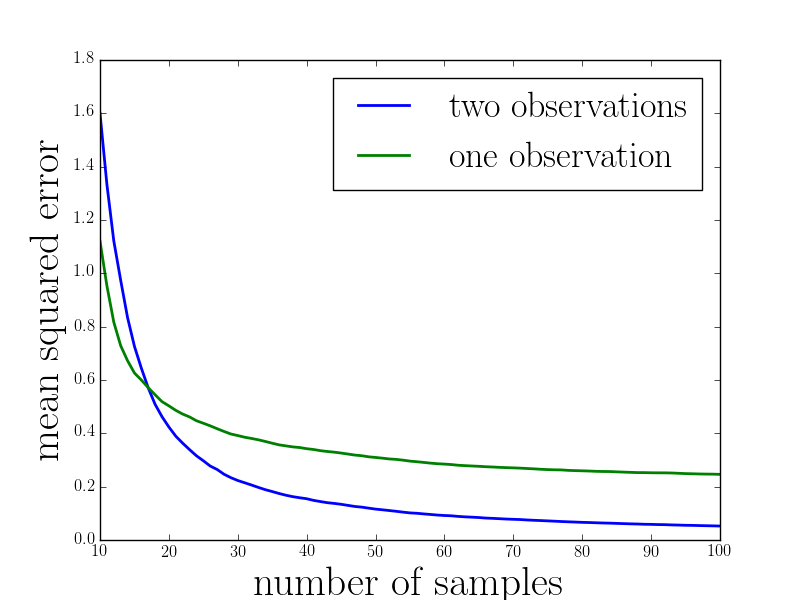}
	\caption{The mean squared error of the two regression strategies for estimating $\Var(y | x)$, where $x \sim \mathrm{Unif(0,1)}$ and $y \sim a \sin(4 \pi x) + N(0,1)$, for $a=1$ (left) and $a=10$ (right).
	  The single-observation loss function approach fails because it tries to fit to the complex underlying model of $y|x$, while the two-observation loss approach is able to directly model the simple relationship between $\Var(y)$ and $x$.}
	\label{fig:variance_reg_linear}
\end{figure}
\else
\fi

Figure \ref{fig:variance_reg_linear} gives the results for $a = 1$ and $a = 10$.  Both plots show the mean squared error of the variance functions reported by the two regression methods (averaged over 4000 simulations) as a function of the number of samples.  In both cases we see that for sufficiently many samples, the two observation regression significantly outperforms the single observation regression.

\section{Conclusion and Future Work}
\label{sec:conclusion}
An immediate host of directions is the proving of upper and lower bounds on elicitation frontiers for various properties.
In particular, our lower bounds here focus on techniques for lower-bounding observation complexity (the $(1,m)$ case), leaving open approaches for lower bounds on $(d,m)$ complexity for $d \geq 2$.
Another direction is to formalize learning guarantees for multi-observation regression under suitable assumptions on slow-changing conditional distributions.

\acks{We thank Karthik Kannan for contributing the upper bound for central moments.  Sebastian Casalaina-Martin was partially supported by NSA grant H98230-16-1-0053.  Tom Morgan was funded in part by NSF grants CCF-1320231 and CNS-1228598. Bo Waggoner is supported by the Warren Center for Network and Data Sciences at the University of Pennsylvania.}

\bibliography{refs}

\newpage

\appendix

\renewtheorem{theorem}{Theorem}[section] 
\renewtheorem{fact}{Fact}[section]
\renewtheorem{claim}{Claim}[section]
\renewtheorem{example}{Example}[section]

\section{Overlapping Level Sets: Proof of Theorem~\ref{thm:direct_lower}}
\label{sec:overlapping}

Theorem~\ref{thm:direct_lower} states that a property is not elicitable if there is a convex combination of one of its level sets in the $m$-product space that equals a convex combination of another one of its level sets in the $m$-product space.  To reason about these level sets we will need the following theorem.

\begin{theorem}[Theorem 3.5, \cite{frongillo2014general}] \label{thm:raf}
The property $\Gamma:\P'\to\reals$ (where $\P' \subseteq \Delta_{\O'}$) is directly elicitable by the loss function $\loss$ if and only if there exists some convex $G : \conv(\P') \rightarrow \bar{\reals}$ with $G(\P') \subseteq \reals$, some $D \subseteq \delta G$, and some bijection $\phi : \Gamma(\P') \rightarrow D$ with $\Gamma(p) = \phi^{-1}(D \cap \delta G_p)$, such that for all $r \in \reals$ and $\o \in \O'$,
$$\ell(r, \o) = \phi(r)(p_r-\o) - G(p_r),$$
where $\{p_r\} \subseteq \P'$ satisfies $\hat{r} = \Gamma(p_{\hat{r}})$ for all $\hat{r}$.
\end{theorem}

Here $\delta G_r$ is the set of subgradients to $G$ at $r$.

\begin{proof}\textbf{of Theorem \ref{thm:direct_lower}}
\footnote{An alternate proof can also be constructed using results of \citet{osband1985providing}.}
Let $\O' = \O^m$ and $\P' = \P^m$.  Let $\loss$ be a loss function that elicits $\Gamma$ of the form given by Theorem \ref{thm:raf}, and let $G,\{p_r\}$ and $\phi$ be the corresponding values defined in Theorem \ref{thm:raf}.  We will let $\Gamma':\P^m\toto\reals$ be the property that is elicited by $\loss$ on $\conv(\P^m)$.

Note that $\Gamma'$ is not necessarily single-valued everywhere on $\conv(\P^m)$.  This is because we cannot guarantee that there is a unique value that minimizes the loss function for distributions in the interior of $\conv(\P^m)$.  However, we can show that whenever $q \in \conv(P^m)$ can be written as a convex combination of points on $\P^m$ that all have property value $r$ then $\E_{\o \sim q} \loss(r, \o)$ is uniquely minimized at $r$, thus $r$ is the unique property value of $\Gamma'$ at $q$.  This implies the theorem, as if $q$ can be written as a convex combination of two separate level sets of $\Gamma$ then there must not be an $\loss$ of the form specified in Theorem \ref{thm:raf} which elicits it.

If $q = \sum_{i=1}^k \lambda_i p^m_i$ for $p \in \Gamma_{r^*}$, $\lambda_1,\ldots,\lambda_k \in [0,1]$ and $\sum_{i=1}^k \lambda_i =1$ then
\begin{align*}
\E_{\o \sim q} \loss(r, \o) &= \phi(r)(q_r-q) - G(q_r)\\
&= \phi(r)\left(q_r-\sum_{i=1}^k \lambda_i p^m_i\right) - G(q_r)\\
&= \sum_{i=1}^k \lambda_i \left(\phi(r)(q_r- p^m_i) - G(q_r)\right)\\
&= \sum_{i=1}^k \lambda_i \E_{\o \sim p^m_i} L(r, \o).
\end{align*}
We know that each term of the final sum is uniquely minimized by $r = r^*$, thus $\E_{\o \sim q} \loss(r, \o)$ is uniquely minimized by $r^*$.
\end{proof}

\section{Regression}
\label{app:regression}
\newcommand{\risk}{\mathrm{Risk}}
\newcommand{\ER}{\mathrm{Risk}_{\mathrm{emp}}}
\newcommand{\ERc}{\mathrm{Risk}_{\mathrm{cl}}}
In this section, we give a proof-of-concept showing that classic risk bounds for ERM can go through with only slight modification with multi-observation loss functions, under a natural assumption.

Regression can be naturally formulated in the multi-observation setting as follows: Given a hypothesis class $\F: \X \to \R$ and loss function $\loss: \R \times \Y^m \to \reals$, given access to an unknown distribution $\D$ on $\X$ and conditional distributions $\{\D_x \in \Delta_{\Y} : x \in \X\}$, approximately minimize
  \[ \risk(f) = \E_{x \sim \D, \vec{\o} \sim \D_x} \loss(f(x), \o_1,\dots,\o_m) . \]
The central challenge that arises, new to the multi-observation setting, is that the data we are given is of the form $(x_1,y_1), \dots, (x_n,y_n)$ where $x_i \sim \D$ and $y_i \sim \D_{x_i}$ i.i.d.
We may only obtain a single $y$ for any given $x$.
In this section, we give an example of how this obstacle can be overcome under natural assumptions.

For simplicity, let us suppose that $\X \subseteq \reals^d$ (in this section, $d$ is not being used for dimensionality of the report space).
The key idea is that, if the distribution $\D_x$ changes slowly as a function of $x$, then with enough samples, then a set of $m$ close neighbors $x_1,\dots,x_m$ can be viewed as approximating a single $x$ with $m$ ``almost i.i.d.'' conditional draws $\o_1,\dots,\o_m$.
We formalize this intuition here using a Lipschitz condition on the total variation distance:
  \[ D_{TV}(\D_x, \D_{x'}) \leq K \|x - x'\|_2 . \]
However, the exact formalization is less important than the general idea, and we expect that future work will be able to prove similar results with a variety of similar assumptions.

Our approach will be to cluster the data into groups of size $m$ having nearby $x$s, then treat each group as a single sample of the form $(x^*, \o_1,\ldots,\o_m)$ with each $\o_i$ \emph{approximately} i.i.d. from $\D_{x^*}$.
We then have $n'$ ``samples'' of this form, where $n'$ is the number of clusters.
Of course, for this approach, it is necessary that that $m$ be small compared to the total number of samples $n \approx n' m$; we are often interested in the $m=2$ case where our theory and simulations already show dramatic differences from the traditional case of $m=1$.

A classic risk bound translated into our setting is the following, where $R_n$ denotes the \emph{Rademacher complexity} of a hypothesis class.
\begin{theorem}[\citet{bartlett2002rademacher}] \label{thm:bartlett}
  Suppose $\loss$ is $L$-Lipschitz in its first argument and bounded by $c$, $\{x_i\}_{i=1}^n$ are drawn i.i.d. from a distribution $\D$, and each $\vec{\o}_i$ is drawn independently from $\D_{x_i}$.
  Then with probability at least $1-\delta$, for all $f \in \F$,
  \[ \risk(f) \leq \ER(f, \{x_i,\vec{\o}_i\}_{i=1}^n) ~ + ~ 2L R_n(\F) ~ + ~ c\sqrt{\frac{\log 1/\delta}{2n}} . \]
  Here the probability is over the randomness in $\{x_i,\vec{\o}_i\}$.
\end{theorem}
In other words, if we could actually sample a set $\vec{\o}_i = (\o_{i,1},\dots,\o_{i,m})$ from $\D_{x_i}$ i.i.d., we would reduce to the standard setting.
This theorem is leveraged to prove specific ERM risk bounds depending on $\F$.
Here we just show that this bound changes only slightly in the multi-observation case, with an increase in sample complexity.

Our ``cluster-points'' algorithm roughly functions as follows: draw $n$ i.i.d. data points $x_1^*,\dots,x_n^*$ and $n' = \Omega(n(m+\log(n/\delta))/\epsilon)$ ``scatter points'' of the form $(x,\o)$.
Assign to each $x_i^*$ a set $\vec{\o}_i^*$ of size $m$ where for each $\o_{ij}^*$, its corresponding $x$ has $\|x - x_i^*\|_2 \leq \epsilon$.
We first show that this is possible with probability $1-\delta$, in two lemmas.

\begin{lemma}
\label{lem:chernoff}
Given $x \in [0,1]$, $\epsilon < 1$ and $\Omega((m + \log(1/\delta'))/\epsilon)$ i.i.d. from the uniform distribution over $[0,1]$, with probability at least $1-\delta'$, at least $m$ of the samples fall within $\epsilon$ of $x$.
\end{lemma}

\begin{proof}
The probability that a given sample falls within $\epsilon$ of $x$ is at least $\epsilon$.  If we take $s$ samples, then by a standard Chernoff bound we have that the probability of fewer than $m$ samples falling within $\epsilon$ of $x$ is upper bounded by
$$e^{-\left(1-\frac{m}{\epsilon s}\right)^2 \epsilon{s} / 2}.$$
Solving for $s$ when this is $\delta'$ gives us the Lemma.
\end{proof}

\begin{lemma} \label{lem:regress-cluster-alg}
Let $\D$ be the uniform distribution on $[0,1]$.  $n' = O(n(m + \log(n / \delta)) /\epsilon)$ samples of the form $(x,y)$ where $x \sim \D$ and $y \sim D_x$ are sufficient to find, with probability at least $1 - \delta$, a set of $n$ independent samples of the form $(x^*,y^*_1,\ldots,y^*_m)$ where $x^* \sim \D$ and the $y^*_i$s are independent and of the form $y^*_i \sim \D_{x'}$ for $|x' - x^*| \leq \epsilon$.  
\end{lemma}

\begin{proof}
First we take $m$ samples and use there $x$ values as our $m$ $x^*$s.  For each $x^*$, we take a new set of $n'/m = O((m + \log(n / \delta)) /\epsilon)$ samples $(x_1,y_1),\ldots,(x_{n'/m},y_{n'/m})$.  Let $j_1,\ldots,j_m$ be $m$ distinct indices such that for all $i$, $|x_{j_i} - x^*| \leq \epsilon$.  By Lemma \ref{lem:chernoff} (setting $\delta' = \delta / n$) such a set will exist with probability at least $1 - \delta / n$.  We then construct the sample
$$(x^*,y^*_1,\ldots,y^*_m) = (x^*, y_{j_1},\ldots,y_{j_m}).$$

By a union bound, this algorithm will succeed with probability at least $1-\delta$, and the produced samples trivially fulfill the distributional requirements of the Lemma.
\end{proof}

Now we obtain the desired result.
Note that we can choose $\epsilon$ as small as desired, e.g. $\epsilon = 1/n^2$, with a blowup of $1/\epsilon$ in the sample complexity.
However, a more sophisticated bound would preferably use higher-powered concentration inequalities or a more carefully tailored assumption in order to get a bound holding with higher probability.
\begin{theorem} \label{thm:erm-regression}
  Suppose $\loss$ is $L$-Lipschitz in its first argument and bounded by $c$, $\D$ is uniform on $[0,1]$, and $\{x_i^*, \vec{\o}_i\}_{i=1}^n$ are drawn according to our cluster-points algorithm, taking $n' = O((m+\log(n/\delta))/\epsilon)$ total samples.
  Then with probability at least $1-2\delta - mnK\epsilon$, for all $f \in \F$,
  \[ \risk(f) \leq \ER(f, \{x_i^*,\vec{\o}_i\}_{i=1}^n) ~ + ~ 2L R_n(\F) ~ + ~ c\sqrt{\frac{\log 1/\delta}{2n}}. \]
  Again the probability is over the randomness in $\{x_i^*,\vec{\o}_i\}$.
\end{theorem}
\begin{proof}
  With probability $1-\delta$, our ``cluster-points'' algorithm succeeds in finding $\{x_i^*\}_{i=1}^n$ drawn i.i.d. and $\{\vec{\o}_i\}_{i=1}^n$ drawn from $\epsilon$-close points.
  We wish to consider $\ER(f, \{x_i^*,\vec{\o}_i\}_{i=1}^n)$, where each $\vec{\o}_i$ is $Km\epsilon$-close in total variation distance to $\vec{\o}_i^*$, as each member is $K\epsilon$ close.
  So the whole quantity, by the properties of total variation distance, is $mnK\epsilon$-close to $\ER(f, \{x_i,\vec{\o}_i\}_{i=1}^n)$, and we apply Theorem \ref{thm:bartlett}.
\end{proof}

\section{Zero sets of Polynomials over the Real Numbers}
\label{sec:algebra-overview}

Consider a polynomial $f(x_1,\dots,x_n)$ in the set $\mathbb R[x_1,\dots,x_n]$  of polynomials   in $n$ variables with real coefficients.    The \emph{zero set} of   $f(x_1,\dots,x_n)$ is by definition  the set
$$
Z(f(x_1,\dots,x_n)):=\{(\alpha_1,\dots,\alpha_n)\in \mathbb R^n: f(\alpha_1,\dots,\alpha_n)=0\}\subseteq \mathbb R^n.
$$

Recall that a   nonconstant polynomial $f(x_1,\dots,x_n)\in \mathbb R[x_1,\dots,x_n]$ is said to be \emph{irreducible} if  it cannot be written as the product of two polynomials in $\mathbb R[x_1,\dots,x_n]$ of strictly lower degree.   Recall also that a subset $  U\subseteq \mathbb R^n$ is said to be \emph{open in the Euclidean topology} if for every $\alpha=(\alpha_1,\dots,\alpha_n)\in U$, there exists a real number  $\epsilon_\alpha>0$, depending on $\alpha$,  such that the ball of radius $\epsilon_\alpha$ centered at $\alpha$, $B_{\epsilon_\alpha}(\alpha_1,\dots,\alpha_n)$, is contained in $U$:
$$B_{\epsilon_\alpha}(\alpha_1,\dots,\alpha_n):=\left\{(\beta_1,\dots,\beta_n) \in \mathbb R^n: \sqrt{(\beta_1-\alpha_1)^2+\cdots +(\beta_n-\alpha_n)^2}<\epsilon_\alpha \right\}\subseteq U.
$$
With this terminology, we can state the following  theorem:

\begin{theorem}\label{T:E1}
Suppose that $f(x_1,\dots,x_n)\in \mathbb R[x_1,\ldots,x_n]$ is a nonconstant irreducible polynomial, and  $U\subseteq \mathbb R^n$ is an open   subset in the Euclidean topology.  If there is a point $$(\alpha_1,\dots,\alpha_n) \in Z(f(x_1,\dots,x_n))\cap U\subseteq \mathbb R^n$$ such that  
\begin{equation}\label{E:SmNullCond}
\left(\frac{\partial f}{\partial {x_1}}(\alpha_1,\dots,\alpha_n),\dots,\frac{\partial f}{\partial {x_n}}(\alpha_1,\dots,\alpha_n)\right)\ne (0,\dots,0)\in \mathbb R^n,
\end{equation}  
 then there are no nonzero polynomials   of degree less than the degree  of $f(x_1,\ldots,x_n)$  that vanish at every point of the zero set $Z(f(x_1,\dots,x_n))\cap U$.  
\end{theorem}

We expect the theorem is   well known; for instance, the case where  $U=\mathbb R^n$ is a special case of  \cite[Thm.~4.5.1]{BCR}.  
The proof of  \cite[Thm.~4.5.1]{BCR} easily generalizes to  our situation.  
For the convenience of the reader, in Theorem \ref{T:1G} below  we include a generalization of  \cite[Thm.~4.5.1]{BCR} that impiles Theorem \ref{T:E1}.

\begin{remark}[Checking the conditions of Theorem \ref{T:E1}]\label{remark:check-conditions}
There are many techniques for checking that a polynomial $f(x_1,\dots,x_n)\in \mathbb R[x_1,\dots,x_n]$ is irreducible and satisfies the condition  \eqref{E:SmNullCond} for all $(\alpha_1,\dots,\alpha_n)\in Z(f(x_1,\dots,x_n))\cap U$, and therefore satisfies the hypotheses of Theorem \ref{T:E1}.  For $n\ge 2$, we recall the following elementary condition that suffices.    Suppose $f(x_1,\dots,x_n)$ is a nonconstant polynomial of degree $d$.  The homogenization of $f(x_1,\dots,x_n)$ is the degree $d$ homogeneous (all monomials of degree $d$)  polynomial  $F(X_0,X_1,\dots,X_n)\in \mathbb R[X_0,\dots,X_n]$  that is obtained from $f(x_1,\dots,x_n)$ by replacing $x_i$ with $X_i$ for $i=1,\dots,n$, and then multiplying each monomial by a power of $X_0$ until it is of degree $d$.  For instance, if $f(x_1,x_2)=x_1^2+2x_2+3$, then $F(X_0,X_1,X_2)=X_1^2+2X_0X_2+3X_0^2$.      \emph{If the complex zero set 
\begin{equation}\label{E:SmNullCondCC}
\left\{(\alpha_0,\dots,\alpha_n)\in \mathbb C^{n+1}: \frac{\partial F}{\partial {X_0}}(\alpha_0,\dots,\alpha_n)=\cdots=\frac{\partial F}{\partial {X_n}}(\alpha_0,\dots,\alpha_n)=(0,\dots,0)\right\}\subseteq \mathbb C^{n+1}
\end{equation}
is equal to $\{(0,\dots,0)\}$ or $\emptyset$, then $f(x_1,\dots,x_n)$ is irreducible and satisfies \eqref{E:SmNullCond} for all $(\alpha_1,\dots,\alpha_n)\in Z(f(x_1,\dots,x_n))$}.   This is by no means a necessary condition for $f(x_1,\ldots,x_n)$ to satisfy the conditions of Theorem \ref{T:E1}, but it is easy to implement in examples.
  There are a number of other techniques that can  be used, including using computer algebra systems.  
\end{remark}

Using the technique outlined in the remark, and standard results in algebraic geometry, it is elementary to establish the following corollary:

\begin{corollary}\label{C:f_c} Let $n\ge 2$, let $U\subseteq \mathbb R^n$ be a nonempty open subset in the Euclidean topology, let $f(x_1,\dots,x_n)\in \mathbb R[x_1,\dots,x_n]$, and for each $c\in \mathbb R$ define $$f_c(x_1,\dots,x_n):=f(x_1,\dots,x_n)+c.$$
Let $F_{c}(X_0,\cdots,X_n)$ be the homogenization of $f_c(x_1,\dots,x_n)$.  

If for some $c_0\in \mathbb R$ the complex zero set \eqref{E:SmNullCondCC} for  $F_{c_0}(X_0,\dots,X_n)$ is equal to $\{(0,\dots,0)\}\subseteq \mathbb C^{n+1}$ or $\emptyset$, then there is a nonempty open subset $B\subseteq  \mathbb R$ in the Euclidean topology such that for all $c\in B$, there are no nonzero polynomials of degree less than $d$ that vanish at every point of the zero set $Z(f_c(x_1,\dots,x_n))\cap U$.
\end{corollary}

As a consequence:

\begin{example}\label{E:f_c}
For a given pair of natural numbers $n$ and $d$ with $n\ge 2$, suppose that:
\begin{itemize}
\item For $c\in \mathbb R$, we set $f_c(x_1,\dots,x_n):=x_1^d+\cdots +x_n^d+(1-x_1-\cdots -x_n)^d+c$.
\item $U:=\{(\alpha_1,\dots,\alpha_n) \in \mathbb R^n: \alpha_1,\dots,\alpha_n > 0,  \ \sum_{i=1}^n\alpha_i< 1\}$.  
\end{itemize}
There exists a nonempty open subset $B\subseteq  \mathbb R$ in the Euclidean topology such that for all $c\in B$, there are no nonzero polynomials of degree less than $d$ that vanish at every point of $Z(f_c(x_1,\dots,x_n))\cap U$.   

We can confirm this using the approach in Corollary \ref{C:f_c}:
\begin{eqnarray*}
F_c&=&cX_0^d+X_1^d+\cdots +X_n^d+(X_0-X_1-\cdots -X_n)^d\\
\partial_{X_0}F_c&=&cdX_0^{d-1}+d(X_0-X_1-\cdots-X_n)^{d-1},\\
\partial_{X_1}F_c&=&dX_1^{d-1}-d(X_0-X_1-\cdots-X_n)^{d-1},\\
 & \vdots& \\
\partial_{X_n}F_c&=&dX_n^{d-1}-d(X_0-X_1-\cdots-X_n)^{d-1}.
\end{eqnarray*}
To use Corollary \ref{C:f_c}, we need to consider the complex  zero set \eqref{E:SmNullCondCC}:
\[
\left\{(\alpha_0,\dots,\alpha_n)\in \mathbb C^{n+1}: \partial_{X_0}F_c  (\alpha_0,\dots,\alpha_n)=\cdots=\partial_{X_n}F_c (\alpha_0,\dots,\alpha_n)=(0,\dots,0)\right\}\subseteq \mathbb C^{n+1},
\]
and show that for some $c\in \mathbb R$ it is either empty or equal to $\{(0,\dots,0)\}$.
We consider the  case $c=0$.  Under this assumption, we have 
$$
0=\partial_{X_0}F_c=cdX_0^{d-1}+d(X_0-X_1-\cdots-X_n)^{d-1}\iff (X_0-X_1-\cdots-X_n)=0.
$$
Then, assuming  $X_0-X_1-\cdots-X_n=0$, we have for $i=1,\ldots,n$ that
$$
0=\partial_{X_i}F_c=dX_i^{d-1}-d(X_0-X_1-\cdots-X_n)^{d-1}\iff X_i=0.
$$
With this new information, returning to $\partial_{X_0}F_c$, we see that we also must have $$X_0=0.$$  
In other words, the complex zero set is $\{(0,\dots,0)\}\subseteq \mathbb C^n$, so that our example satisfies the conditions of Corollary \ref{C:f_c}.  
\end{example}

\begin{remark}\label{remark:most-polynomials}
Most polynomials $f(x_1,\dots,x_n)\in \mathbb R[x_1,\dots,x_n]$, $n\ge 2$, satisfy the hypotheses of  Corollary \ref{C:f_c}.  More precisely, 
there is a dense open subset (the complement of linear subspace) of an 
$\binom{n+d}{d}$-dimensional real vector space that parameterizes degree-$d$ polynomials in $n$ variables.  
That subset contains a dense open subset $\Omega$ (the complement of  the discriminant locus; see e.g., \cite{fulton}) such that every $f(x_1,\dots,x_n)\in \Omega$ satisfies the hypotheses of the corollary; i.e., there is  some $c_0\in \mathbb R$ (for instance $c_0=0$) such that the complex zero set \eqref{E:SmNullCondCC} for  $F_{c_0}(X_0,\dots,X_n)$ is equal to $\{(0,\dots,0)\}\subseteq \mathbb C^{n+1}$ or $\emptyset$.
   On the other hand, as described in Example  \ref{E:f_c2} below, it is easy to find  polynomials  $f(x_1,\dots,x_n)\in \mathbb R[x_1,\dots,x_n]$ of degree $d\ge 2$,  and nonempty open subsets $U\subseteq \mathbb R^n$, so that  for every $c\in \mathbb R$
   there exist nonzero polynomials of degree less than $d$ that vanish at every point of the zero set  $Z(f_c(x_1,\dots,x_n))\cap U$.
\end{remark} 
   
\begin{example}\label{E:f_c2}
Consider the polynomial $f(x_1,\dots,x_n)=x_1^2$, and take $U=\mathbb R_{>0}\times \mathbb R^{n-1}$.  Then for every $c\in \mathbb R$ there is a linear polynomial that vanishes  at every point of  $Z(f_c(x_1,\dots,x_n))\cap U$; for $c>0$, we can take any linear polynomial, and for $c\le 0$, we can take $x_1-\sqrt{-c}$.  

  We can construct many more similar examples in the following way.  Let $h(x_1)\in \mathbb R[x_1]$ be a polynomial of degree at least $2$. We have for every $c\in \mathbb R$ that $h(x_1)+c$ factors in $\mathbb R[x_1]$ as a product of linear terms and a product of quadratic terms each having no real root.
  For simplicity, let us assume that for all $c\ne 0$,  the polynomial $h(x_1)+c$ has a root that is not real; e.g., $h(x_1)=x_1^m$ for some natural number $m\ge 3$.  
     Let $g(x_1,\dots,x_n)\in \mathbb R[x_1,\dots,x_n]$ be \emph{any} nonconstant polynomial.   
      Let $\lambda$ be a real root of $h(x_1)$ (if there is one), and let $U$ be the complement of the  zero set of $g(x_1,\dots,x_n)-\lambda$, or simply $\mathbb R^n$ if there is no real root.   Then $f(x_1,\dots,x_n):=h(g(x_1,\dots,x_n))$ has the property that for every $c\in \mathbb R$, there is a nonzero polynomial of degree less than the degree of $f(x_1,\dots,x_n)$ that vanishes at every point of the zero set  $Z(f_c(x_1,\dots,x_n))\cap U$.
\end{example}   

\begin{remark}
Theorem \ref{T:E1} and Corollary \ref{C:f_c} are not interesting in the case $n=1$. For $f(x)\in \mathbb R[x]$ (irreducible or not) there are no nonzero polynomials of degree less than $d$ that vanish at every point of the zero set   $Z(f(x))\cap U$ if and only if \emph{all of the roots of $f(x)$ are real, distinct, and lie in $U$}.
There are standard techniques to check this condition (e.g.,   \cite[pp.12--14]{BCR}).   In  Example \ref{E:f_c} with $n=1$, by inspection one finds that for  $d=1,2$ the condition   holds if and only if $-1<c<1$, and for $d=3,4$ the condition  does not hold for any $c$.   
\end{remark}

\section{The Real Nullstellenstatz for Principal Ideals and Open Sets}
\label{sec:algebra-theorem}

The main goal of this section is to prove the following theorem  generalizing    the well known  real Nullstellenstatz for principal ideals 
(e.g., \cite[Thm.~4.5.1]{BCR}) to allow for Euclidean open sets.

\begin{theorem} \label{T:1G} Let $\mathbb K$ be a real closed field (e.g., $\mathbb K=\mathbb R$).  
 Let $f(x_1,\dots,x_n)\in \mathbb K[x_1,\ldots,x_n]$ be a nonconstant polynomial, and 
 let $U\subseteq \mathbb K^n$ be an open   subset in the Euclidean topology.
 Suppose that 
 \begin{equation}\label{E:ffactor}
f(x_1,\dots,x_n)=f_1(x_1,\dots,x_n)^{m_1}\cdots f_r(x_1,\dots,x_n)^{m_r}
\end{equation}
is a factorization into powers of distinct nonconstant irreducible polynomials.
   The following are equivalent:
\begin{enumerate}

\item  $(f) = I(Z(f)\cap U)$.

\item   $m_1=\cdots=m_r=1$ and for each $i=1,\dots,r$ there is a point $\alpha^{(i)}\in Z(f_i)\cap U$ with $$(\partial_{x_1}f_i(\alpha^{(i)}),\dots,\partial_{x_n}f_i(\alpha^{(i)}))\ne 0\in \mathbb K^n.$$  
For $\mathbb K=\mathbb R$, this  is equivalent to having for each $i$ that  
 $Z(f_i)\cap U$ is a smooth $(n-1)$-dimensional  submanifold of an open neighborhood of  $\alpha^{(i)}$.
 
\item  $m_1=\cdots=m_r=1$ and for each $i=1,\dots,r$  the sign of the polynomial $f_i$ changes on an open ball in $U$ (i.e., for  $i=1,\dots,n$ there is an open ball $B_\epsilon^{(i)}\subseteq U$ and points  $\alpha^{(i)},\beta^{(i)}\in B_\epsilon^{(i)}$ such that  $f_i(\alpha^{(i)})f_i(\beta^{(i)}) < 0$).
 
\item $m_1=\cdots=m_r=1$ and for each $i=1,\dots,r$  the semi-algebraic Krull dimension of the topological space $Z(f_i)\cap U$  (i.e., the Krull dimension of the  ring $\mathbb K[x_1,\dots,x_n]/I(Z(f_i)\cap U)$) satisfies  $$\dim(Z(f_i)\cap U) = n -1.$$
\end{enumerate}
\end{theorem}

We expect this result  is known to the experts (the case where $f$ is irreducible and   $U=\mathbb R^n$ is \cite[Thm.~4.5.1]{BCR}), but  for lack of a  reference  we provide a proof  in \sect \ref{S:T1Gpf}.  See \sect \ref{S:notation} for an explanation of the notation.

\begin{remark}
The case $n=1$  is elementary and has the following simple interpretation:  \emph{we have  $(f(x)) = I(Z(f(x))\cap U)$ if and only if all of the roots of $f(x)$ in an algebraic closure $\overline {\mathbb K}$ are  distinct, and lie in $U\subseteq \mathbb K$.}    There are standard techniques to check this condition (e.g.,   \cite[pp.12--14]{BCR}).
\end{remark}

\begin{remark}
If $f(x_1,\dots,x_n)$ is given as in \eqref{E:ffactor}, then the radical of the ideal $(f)$ is the ideal  $\sqrt{(f)}=(f_1\cdots f_r)$.  Thus Theorem \ref{T:1G} also gives conditions   for when there is an equality $\sqrt{(f)}=I(Z(f)\cap U)$.  
\end{remark}

\subsection{Notation and conventions}\label{S:notation}
Let $K$ be a field.  Given an ideal $I\subseteq K[x_1,\dots,x_n]$ we will be interested in both the closed subscheme
$$
V(I)\subseteq \mathbb A^n_K,
$$
as well as 
 the zero set
$$
V(I)(\operatorname{Spec} K)\simeq Z_K(I):=\{\alpha\in K^n: f(\alpha)=0,\text{ for all } f\in I\}\subseteq K^n. 
$$
If the field is clear from the context, we will write $Z(I)=Z_K(I)$.   
For a subset $S\subseteq K^n$, we denote as usual the ideal of polynomials vanishing on $S$ as
$$
I(S):=\{g(x_1,\dots,x_n)\in   K[x_1,\dots,x_n] : g(s)=0  \ \text{ for all } \ s\in S\}.
$$

We refer the reader to \cite[Def.~1.1.9, Def.~1.2.1]{BCR} for a review of the definition of a real closed field.  In particular, such a field $\mathbb K$ is of characteristic $0$ and is an ordered field; the Euclidean topology on $\mathbb K^n$ then has a basis given by the open balls 
$$
B_\epsilon(\alpha):=\{\beta\in \mathbb K^n: \sum_{i=1}^n(\beta_i-\alpha_i)^2<\epsilon^2\}
$$
for all $\alpha \in \mathbb K^n$ and all $\epsilon\in \mathbb K$ with $\epsilon >0$.  

\subsection{The principal Nullstellensatz}  For an ideal $I\subseteq K[x_1,\dots,x_n]$, there is a natural inclusion
\begin{equation}\label{E:NSSInc}
\sqrt I\subseteq I(Z(I)).
\end{equation}
Hilbert's Nullstellensatz asserts that over an algebraically closed field $\overline K=K$, this inclusion   is an equality.  
Focusing on principal ideals, this reads 
\begin{equation}\label{E:NSSInc2}
 \ \ \ \  \sqrt{(f)}=I(Z(f)), \ \ \ \ \  (K=\overline K);
\end{equation}
in other words $(f) =I(Z(f))$ whenever $f$ is reduced and $K=\overline K$ is algebraically closed.  

Over nonalgebraically closed fields \eqref{E:NSSInc2} clearly fails; i.e., one may have
$$\sqrt{(f)}\subsetneq I(Z(f)).$$
 For instance, trivially, one has in $\mathbb Q[x]$ that $\sqrt {(x^2+1)}=(x^2+1)\subsetneq \mathbb Q[x]=I(\emptyset)=I(Z(x^2+1))$.   
The following example is a little more interesting:

\begin{example}\label{E:2}  Consider $f(x,y)=x^2+y^2-x^3\in \mathbb R[x,y]$, and the zero set $Z(f)\subseteq \mathbb R^2$.  
   It is a cubic plane curve  with an isolated point at $(0,0)\in \mathbb R^2$. In particular, if we take $U=B_\epsilon(0,0)$ to be a small ball around $(0,0)$ in $\mathbb R^2$, then $\sqrt{ (x^2+y^2-x^3)}= (x^2+y^2-x^3) \ne (x,y)= I(Z(x^2+y^2-x^3)\cap U)$.     On the other hand, it is true that $(x^2+y^2-x^3)=I(Z(x^2+y^2-x^3))$.
\end{example}

\subsection{The connection with dimension}

\begin{proposition} \label{P:K}
 Let $f(x_1,\dots,x_n)\in K[x_1,\ldots,x_n]$ be a nonconstant irreducible polynomial, and 
 let $U\subseteq   K^n$ be \emph{any} subset.
   The following are equivalent:
\begin{enumerate}

\item  $(f) = I(Z(f)\cap U)$.

\item The semi-algebraic Krull dimension of the topological space $Z(f)\cap U$  (i.e., the Krull dimension of the ring  $K[x_1,\dots,x_n]/I(Z(f)\cap U)$) satisfies  $$\dim(Z(f)\cap U) = n -1.$$
\end{enumerate}
\end{proposition}

\begin{proof}
(1) $\implies$ (2).   By assumption we have $(f)= I(Z(f)\cap U)$.
Now  the Krull dimension of $K[x_1,\dots,x_n]$ is $n$ (e.g., \cite[Exe.~11.7]{AM69}).  Consequently, since $f$ is neither a zero divisor nor a unit, we have that the Krull dimension of $K[x_1,\dots,x_n]/(f)$ is $(n-1)$ (e.g., \cite[Cor.~11.7]{AM69}; using that $f$ is irreducible, this is even easier).   Note that this direction does not require that $f$ be irreducible.

(2) $\implies$ (1). 
We have inclusions  
\begin{equation}\label{E:PK}
 (f)\subseteq I(Z(f)\cap U)\subseteq K[x_1,\ldots,x_n].
\end{equation}
 As above, since $f$ is neither a zero divisor nor a unit, we have that the Krull dimension of the ring $K[x_1,\dots,x_n]/(f)$ is $(n-1)$.  
By assumption, the Krull dimension of $K[x_1,\dots,x_n]/I(Z(f)\cap U)$  is also $(n-1)$.  
Now since $(f)$ is prime (finally using that $f$ is  irreducible), and has the same Krull dimension as the ideal $I(Z(f)\cap U)$, it follows from the containment \eqref{E:PK} and  the definition of Krull dimension that the two ideals are equal.
\end{proof}

\subsection{The connection with smoothness}

We say a zero set $Z(I)\subseteq K^n$ is smooth at a point $\alpha\in Z(I)$ if the associated scheme $V(I)\subseteq \mathbb A^n_K$ is smooth at the point $(x_1-\alpha_1,\dots,x_n-\alpha_n)\in V(I)$.    We will  also simply say that $V(I)$ is smooth at $\alpha$.  Recall that if $I=(f)$ is principal, and $\alpha\in Z(f)$, then $V(f)$ is smooth at $(x_1-\alpha_1,\dots,x_n-\alpha_n)$ if and only if 
$(\partial_{x_1}f(\alpha ),\dots,\partial_{x_n}f(\alpha ))\ne 0\in K^n$.  

\begin{lemma} \label{L:NullSmooth}  Suppose $K$ is perfect.  Let $f(x_1,\dots,x_n)\in K[x_1,\ldots,x_n]$ be a nonconstant  polynomial, and 
 let $U\subseteq   K^n$ be \emph{any} subset.
  Then:
\begin{enumerate}

\item  $(f) = I(Z(f)\cap U)$,

\end{enumerate}
implies
\begin{enumerate}
 
\item[(2)] 
   There is a point $\alpha^{(0)}\in Z(f)\cap U$ with $$(\partial_{x_1}f(\alpha^{(0)}),\dots,\partial_{x_n}f(\alpha^{(0)}))\ne 0\in K^n.$$  
  In other words, there is a point in $U$ at which $V(f)$ is a smooth scheme.

\end{enumerate}
\end{lemma}

\begin{proof}
  We will show the contrapositive.  Suppose that (2) fails.  
This means that $\partial_{x_1}f, \dots, \partial_{x_n}f\in I(Z(f)\cap U)$.
But since $f$ is nonconstant and $K$ is perfect,  either    there is an  $i$ such that $\partial_{x_i}f$ is nonzero, or  $\operatorname{char}(K)=p>0$ and there exists a polynomial $g\in K[x_1,\ldots,x_n]$ such that $f=g^p$ (e.g., \cite[Ch.~9 Ex.~10, p.524]{CLO}).  In the first case, 
since  $\partial_{x_i}f$ is nonzero of degree less than the degree of $f$,  it cannot be a multiple of $f$, and therefore is not in $(f)$.  Thus  $(f) \subsetneq I(Z(f)\cap U)$, and  (1) fails. 
 In the second case, where $f=g^p$, we have   $g\in  I(Z(f)\cap U)$, while $g\notin (f)$, again for degree reasons,  so that    (1) also fails in this case.
\end{proof}

The following example shows that the converse to Lemma \ref{L:NullSmooth} need not hold.

\begin{example}
Let $K=\mathbb Q$ and let $f(x_1,x_2)=x_1^3+x_2^3-1$.  Then $Z(f)\subseteq \mathbb Q^2$ is a finite set of points, and in particular one can show that  $(f)\subsetneq I(Z(f))$.   On the other hand, at the point say $(1,0)\in Z(f)$, one has $(\partial_{x_1}f(1,0), \partial_{x_2}f(1,0))=(3,0)\ne 0\in \mathbb Q^2$.  
\end{example}

Nevertheless, a converse to Lemma \ref{L:NullSmooth} does hold over the real and complex numbers.  This is essentially because the implicit function theorem asserts that condition (2) implies that the zero set is   an $(n-1)$-dimensional manifold in a neighborhood of the given point.  In fact, one can also establish the converse over real closed fields:

\begin{lemma} \label{L:SmoothNull}  Suppose $K=\mathbb K$ is real closed or equal to $\mathbb C$.  
 Let $f(x_1,\dots,x_n)\in \mathbb K[x_1,\ldots,x_n]$ be a nonconstant irreducible  polynomial, and 
 let $U\subseteq   \mathbb K^n$ be an open subset in the Euclidean topology.
  Then:
\begin{enumerate}

\item  $(f) = I(Z(f)\cap U)$,

\end{enumerate}
is implied by 
\begin{enumerate}
 
\item[(2)] 
   There is a point $\alpha^{(0)}\in Z(f)\cap U$ with $$(\partial_{x_1}f(\alpha^{(0)}),\dots,\partial_{x_n}f(\alpha^{(0)}))\ne 0\in \mathbb K^n.$$  
  In other words, there is a point in $U$ at which $V(f)$ is a smooth scheme.

\end{enumerate}
\end{lemma}

\begin{proof}
Consider the case $K=\mathbb K$ is real closed.    
Let  $\overline{(Z(f)\cap U)}^{\text{Zar}}\subseteq \mathbb K^n$ be the closure in  the Zariski topology.   Now using  condition (2), and (iii) $\implies$ (ii) of  \cite[Prop.~3.3.10]{BCR}, we have that $\dim \mathbb K[x_1,\dots,x_n]/I(\overline{(Z(f)\cap U)}^{\text{Zar}})=n-1$.    (We are applying  \cite[Prop.~3.3.10]{BCR} with $V=\overline{(Z(f)\cap U)}^{\text{Zar}}$ and $P_1=f$.)
Now we observe that $I(Z(f)\cap U)=I(\overline{(Z(f)\cap U)}^{\text{Zar}})$, and   conclude that $\dim (Z(f)\cap U)=n-1$.  
 Note that so far we did not use that $f$ was irreducible, as this is not required in \cite[Prop.~3.3.10]{BCR}.  
To conclude (1), we use Proposition \ref{P:K}, and the assumption that $f$ is irreducible.  
The case where $K=\mathbb C$ is standard, and can be proven in a similar way. 
\end{proof}

\subsection{The connection with the sign of the polynomial}

\begin{lemma} \label{L:Sign}  Suppose $K=\mathbb K$ is real closed.  
 Let $f(x_1,\dots,x_n)\in \mathbb K[x_1,\ldots,x_n]$ be a nonconstant irreducible  polynomial, and 
 let $U\subseteq   \mathbb K^n$ be an open subset in the Euclidean topology.
  Then the following are equivalent:
\begin{enumerate}

\item  $(f) = I(Z(f)\cap U)$.

\item  The sign of the polynomial $f$ changes on an open ball in $U$ (i.e., there is an open ball $B_\epsilon\subseteq U$ such that  $f(\alpha)f(\beta) < 0$ for some $\alpha,\beta\in B_\epsilon$).

\end{enumerate}
\end{lemma}

\begin{proof}
(1) $\implies$ (2).  Assuming (1), then from Lemma \ref{L:NullSmooth}, there is a point $\alpha^{(0)}\in Z(f)\cap U$ with $(\partial_{x_1}f(\alpha^{(0)}),\dots,\partial_{x_n}f(\alpha^{(0)}))\ne 0\in \mathbb K^n$.
In other words,  there is an $i$ such that $\partial_{x_i}f(\alpha^{(0)})\ne 0$.   Then consider the polynomial in one variable 
$$
\phi(x_i):=f(\alpha_1^{(0)},\dots,x_i,\dots, \alpha_n^{(0)}).
$$
We have  $\phi(\alpha_i^{(0)})=0$.  But since   $\phi'(\alpha_i^{(0)})=\partial_{x_i}f(\alpha_i^{(0)})$ is non-zero, the function $\phi(x_i)$ is monotone in a  real interval around $\alpha_i^{(0)}$, and so it changes sign  \cite[Cor.~1.2.7]{BCR}.  Therefore $f$ changes sign.  (Note that we did not use that $f$ was irreducible.)

(2) $\implies$ (1).     \cite[Lem.~4.5.2]{BCR} states the following:  \emph{Let $B_\epsilon\subseteq \mathbb K^n$ be an open ball (including the case where $B_\epsilon = \mathbb K^n$) and let $U_1$ and $U_2$ be 
two disjoint nonempty semi-algebraic  open subsets of $B_\epsilon$. Then we have $\dim(B_\epsilon-(U_1\cup U_2)) \ge n- 1$}.  Now apply this in our situation, with 
$$U_1=\{\alpha\in B_\epsilon: f(\alpha)>0\} \ \text{ and }
 \ U_2=\{\alpha\in B_\epsilon: f(\alpha)<0\},$$
 so that $B_\epsilon-(U_1\cup U_2)=Z(f)\cap B_\epsilon$.
Then 
$$
n-1=\dim Z(f)\ge \dim(Z(f)\cap U) \ge \dim (Z(f)\cap B_\epsilon)\ge  n -1.
$$
As mentioned above, we have the equality  $\dim Z(f)=n-1$ on the left since  $f$ is  neither a zero divisor nor a unit.
 Note that so far we did not use that $f$ was irreducible. To conclude (1), we use Proposition \ref{P:K}, and the assumption that $f$ is irreducible.
\end{proof} 

\subsection{Proof of Theorem \ref{T:1G}}\label{S:T1Gpf}

\begin{proof}[Proof of Theorem \ref{T:1G}]

  We have now proved the theorem under the hypothesis that $f$ is irreducible (Proposition \ref{P:K}, Lemma \ref{L:NullSmooth}, Lemma \ref{L:SmoothNull}, Lemma \ref{L:Sign}).  We now reduce to this case.

First, it is clear that (2) $\iff$ (3) $\iff$ (4), from the irreducible case.  
Also, it is clear that if (1) holds (i.e., $(f)=I(Z(f)\cap U)$), we must have that $m_1=m_2=\cdots =m_r=1$.
Indeed, if say $m_1>1$, then $f_1f_2^{m_2}\cdots f_r^{m_r}\in I(Z(f)\cap U)$, but for degree reasons $f_1f_2^{m_2}\cdots f_r^{m_r}$  is not a multiple of $f=f_1^{m_1}\cdots f_r^{m_r}$ and thus (1) fails.    
So from here on, we assume $m_1=m_2=\cdots =m_r=1$.

(1) $\implies$ (2).   Suppose that (2) fails.  Then there is some  $i,j$ so that $ \partial_{x_j}f_i\in I(Z(f_i)\cap U)$ and is nonzero.   Therefore   $f_1\cdots \partial_{x_j}f_i \cdots f_r\in  I(Z(f)\cap U))$ and is nonzero.  But for degree reasons, it is not a multiple of $f=f_1\cdots f_i\cdots f_r$ and thus (1) fails.  

(2) $\implies$ (1).  
This follows from the fact that 
\begin{align*}
\bigcap_{i=1}^r (f_i)& = (f_1\cdots f_r)& (\mathbb K[x_1,\dots,x_n] \text{ is a UFD})\\
&=(f)\\
&\subseteq I(Z(f)\cap U))\\
&=I\left (\bigcup_{i=1}^r  \left(Z(f_i)\cap U\right)\right)\\
&=\bigcap_{i=1}^r  I(Z(f_i)\cap U),
\end{align*}
since, assuming  (2) and the special case of Theorem \ref{T:1G} for irreducible polynomials,   then for all $i$, we have $(f_i)=I(Z(f_i)\cap U)$, forcing the containment above to be an  equality.  \end{proof}

\end{document}